%% file: main.tex
\PassOptionsToPackage{x11names}{xcolor}
\PassOptionsToPackage{table}{xcolor}
\documentclass{article}
\usepackage[authoryear]{natbib}

\usepackage{geometry}

\usepackage{microtype}
\usepackage{tikz}
\usepackage{subcaption}
\usepackage{mdframed}
\usetikzlibrary{arrows, calc}
\usepackage{amsfonts, amsmath, booktabs, microtype, multirow, enumitem, graphicx}
\usepackage{hyperref}
\usepackage{listings}
\usepackage{wrapfig}
\usepackage{thm-restate}
\usepackage{titletoc}
\usepackage{xcolor}
\usepackage[capitalize,noabbrev,nameinlink]{cleveref}
\usepackage[title]{appendix}
\usepackage[USenglish]{babel}

\definecolor{papercolor}{HTML}{0668E1}
\definecolor{darkred}{rgb}{0.68,0.05,0.0}

\hypersetup{
    colorlinks,
    linkcolor={papercolor!75!black},
    citecolor={papercolor!75!black},
    urlcolor={papercolor!75!black}
}

\colorlet{alternateRowColor}{papercolor!5}

\newcommand{\coloredBelowRuleSep}[1]{
    \arrayrulecolor{#1}
    \specialrule{\belowrulesep}{0pt}{0pt}
    \arrayrulecolor{black}
}

\newcommand{\coloredMidrule}[2]{
    \arrayrulecolor{#1}
    \specialrule{\aboverulesep}{0pt}{0pt}
    \arrayrulecolor{black}
    \specialrule{\lightrulewidth}{0pt}{0pt}
    \coloredBelowRuleSep{#2}
}

\usepackage{amsthm}
\input{math_commands.tex}

\newcommand{\method}{\textsc{ICRM}}
\newcommand{\methodfull}{In-Context Risk Minimization}

\newcommand{\Eall}{\mathcal{E}}
\newcommand{\Etr}{\mathcal{E}_\text{tr}}
\newcommand{\Ete}{\mathcal{E}_\text{te}}

\usepackage{pifont}

\newtheorem{theorem}{Theorem}
\newtheorem{proposition}{Proposition}
\newtheorem{assumption}{Assumption}
\newtheorem{definition}{Definition}
\newtheorem{lemma}{Lemma}

\begin{document}
\begin{center}
\begin{LARGE}
  \textbf{Context is Environment}
\end{LARGE}
\vskip 1cm
\begin{minipage}[b]{3.5cm}
  \centering
  \textbf{Sharut Gupta\thanks{}}\\
  Meta AI, MIT CSAIL\\
  \texttt{sharut@mit.edu}
\end{minipage}
\hskip .75cm
\begin{minipage}[b]{3cm}
  \centering
  \textbf{Stefanie Jegelka}\\
  MIT CSAIL\\
  \texttt{stefje@mit.edu}
\end{minipage}
\hskip .75cm
\begin{minipage}[b]{6cm}
  \centering
  \textbf{David Lopez-Paz, Kartik Ahuja}\\
  Meta AI\\
  \texttt{\{dlp,kartikahuja\}@meta.com}
\end{minipage}

\vskip 1cm
\end{center}

\begin{abstract}
  Two lines of work are taking the central stage in AI research.
  On the one hand, the community is making increasing efforts to build models that discard spurious correlations and generalize better in novel test environments.
  Unfortunately, the bitter lesson so far is that no proposal convincingly outperforms a simple empirical risk minimization baseline.
  On the other hand, large language models (LLMs) have erupted as algorithms able to learn \emph{in-context}, generalizing on-the-fly to eclectic contextual circumstances that users enforce by means of prompting.
  In this paper, we argue that \emph{context is environment}, and posit that in-context learning holds the key to better domain generalization.
  Via extensive theory and experiments, we show that paying attention to context---unlabeled examples as they arrive---allows 
  our proposed In-Context Risk Minimization (ICRM) algorithm
  to \emph{zoom-in} on the test environment risk minimizer, leading to significant out-of-distribution performance improvements.
  From all of this, two messages are worth taking home.
  Researchers in domain generalization should consider \emph{environment as context}, and harness the adaptive power of in-context learning.
  Researchers in LLMs should consider \emph{context as environment}, to better structure data towards generalization.
\end{abstract}

\footnotetext{Most of the work done during an internship at Meta AI (FAIR), Paris.}
\section{Introduction}

One key problem in AI research is to build systems that generalize across a wide range of test environments.
In principle, these algorithms should discard spurious correlations present only in certain training environments, and capture invariant patterns appearing across conditions. 
For example, we would like to build self-driving systems that, while trained on data from environments with varying weather conditions, traffic conditions, and driving rules, can perform satisfactorily in completely new environments.
Unfortunately, this has so far been a far cry: models trained catastrophically fail to generalize to unseen weather conditions~\citep{lechner2022all}.
Despite its importance, how to perform well beyond the distribution of the training data remains a burning question.
In fact, entire research groups are devoted to study generalization, major international conferences offer well-attended workshops dedicated to the issue~\citep{scis}, and news articles remind us of the profound societal impact from failures of ML systems~\citep{compas}.

Research efforts have so far produced domain generalization algorithms that fall into one out of two broad categories.
On the one hand, invariance proposals~\citep{dann,  peters2016causal, irm}, illustrated in~\Cref{figure:main:a}, discard all environment-specific information, thus removing excessive signal about the problem.
On the other hand, marginal transfer proposals~\citep{marginal, adabn, arm, contextvit}, also illustrated in~\Cref{figure:main:b}, summarize observed inputs in each environment as a coarse embedding, diluting important signal at the example level.
So far, the bitter lesson is that no algorithm geared towards out-of-distribution generalization outperforms a simple empirical risk minimization (ERM) baseline when evaluated across standard real-world benchmarks~\citep{domainbed, gagnon-audet2023woods, yao2022wild}.
Has the generalization project hit a dead end?

\begin{figure}
  \centering
    \begin{subfigure}[b]{0.3\textwidth}
    \centering
    \resizebox{\textwidth}{!}{%
    \begin{tikzpicture}
      \tikzset{font=\footnotesize}
      \node[draw=black, inner sep=0, minimum height=2em, minimum width=2em] (C1) at (0, 0) {$x^e_{1}$};
      \node[inner sep=0, minimum height=2em, minimum width=2em] (C2) at (1, 0) {$\cdots$};
      \node[draw=black, inner sep=0, minimum height=2em, minimum width=2em, fill=papercolor!50] (C3) at (2, 0) {$x^e_{i-1}$};
      \node[draw=black, inner sep=0, minimum height=2em, minimum width=2em] (X)  at (3, 0) {$x^e_i$};
      \node[draw=black, inner sep=0, minimum height=2em, minimum width=2em] (Y)  at (3, 1.8) {$\hat{y}^e_i$};
      \draw [draw=black] (-0.4, 0.55) rectangle (3.4, 1.25) node[pos=.5] {classifier};
      
      \draw [<-] (Y.south)  -- ++(0, -.2);
      \draw [->] (X.north)  -- ++(0, 0.2);
    \end{tikzpicture}}
    \caption{Invariance DG.}
    \label{figure:main:a}
  \end{subfigure}
  \hfill
  \begin{subfigure}[b]{0.3\textwidth}
    \centering
    \resizebox{\textwidth}{!}{%
    \begin{tikzpicture}
      \tikzset{font=\footnotesize}
      \node[draw=black, inner sep=0, minimum height=2em, minimum width=2em] (C1) at (0, 0) {$x^e_{i,1}$};
      \node[inner sep=0, minimum height=2em, minimum width=2em] (C2) at (1, 0) {$\cdots$};
      \node[draw=black, inner sep=0, minimum height=2em, minimum width=2em, fill=papercolor!50] (C3) at (2, 0) {$x^e_{i-1}$};
      \node[draw=black, inner sep=0, minimum height=2em, minimum width=2em] (X)  at (3, 0) {$x^e_i$};
      \node[draw=black, inner sep=0, minimum height=2em, minimum width=2em, fill=papercolor!15] (C)  at (1, 2) {$\phi^e_i$};
      \node[draw=black, inner sep=0, minimum height=2em, minimum width=2em] (Y)  at (3, 3.8) {$\hat{y}^e_i$};
      \draw [draw=black] (-0.4, 2.55) rectangle (3.4, 3.25) node[pos=.5] {classifier};
      
      \draw [->] (C.north) -- ++(0, 0.2);
      \draw [->] (X.north) -- ++(0, 2.2);
      \draw [<-] (Y.south) -- ++(0, -.2);
      
      \draw [->] (C1.north) -- ++(0, 0.2);
      \draw [->] (C2.north) -- ++(0, 0.2);
      \draw [->] (C3.north) -- ++(0, 0.2);
      \draw [<-] (C.south) -- ++(0, -0.2);

      \draw ($(C1.north west) + (0,.2)$) -- %
            ($(C3.north east) + (0,.2)$) -- %
            ($(C.south east)  - (0,.2)$) -- %
            ($(C.south west)  - (0,.2)$) -- %
            ($(C1.north west) + (0,.2)$);
      
      \node[] (sum)  at (1, 1) {\scalebox{0.7}{$\frac{1}{i-1}\sum\limits_{j=1}^{i-1} \phi(x^e_j)$}};
    \end{tikzpicture}}
    \caption{Marginal transfer DG.}
    \label{figure:main:b}
  \end{subfigure}
  \hfill
  \begin{subfigure}[b]{0.3\textwidth}
    \centering
    \resizebox{\textwidth}{!}{%
    \begin{tikzpicture}
      \tikzset{font=\footnotesize}
      \node[draw=black, inner sep=0, minimum height=2em, minimum width=2em] (C1) at (0, 0) {$x^e_{1}$};
      \node[inner sep=0, minimum height=2em, minimum width=2em] (C2) at (1, 0) {$\cdots$};
      \node[draw=black, inner sep=0, minimum height=2em, minimum width=2em, fill=papercolor!50] (C3) at (2, 0) {$x^e_{i-1}$};
      \node[draw=black, inner sep=0, minimum height=2em, minimum width=2em] (X)  at (3, 0) {$x^e_i$};
      \node[draw=black, inner sep=0, minimum height=2em, minimum width=2em] (Y)  at (3, 1.8) {$\hat{y}^e_i$};
      \draw [draw=black] (-0.4, 0.55) rectangle (3.4, 1.25) node[pos=.5] {transformer};
      
      \draw [<-] (Y.south)  -- ++(0, -.2);
      \draw [->] (C1.north) -- ++(0, 0.2);
      \draw [->] (C2.north) -- ++(0, 0.2);
      \draw [->] (C3.north) -- ++(0, 0.2);
      \draw [->] (X.north)  -- ++(0, 0.2);
    \end{tikzpicture}}
    \caption{In-context DG (ours).}
    \label{figure:main:c}
  \end{subfigure}
  \caption{
    Three frameworks for domain generalization (DG), predicting the target $y^e_i$ from the input $x^e_i$ at test environment $e$.
    Depicted in blue, the last example $x^e_{i-1}$ contains relevant features for the current prediction.
    (a) Invariance DG discards all of the previously observed information from the test environment, removing too much predictive signal.
    (b) Marginal transfer DG summarizes all of the previously observed test inputs as a coarse embedding, diluting predictive signal found at the example level.
    (b) Our in-context DG directly observes all of the previous test inputs, allowing the search of ``needle-in-the-haystack'' signals, such as the relevant one in $x^e_{i-1}$.}
  \label{figure:main}
\end{figure}
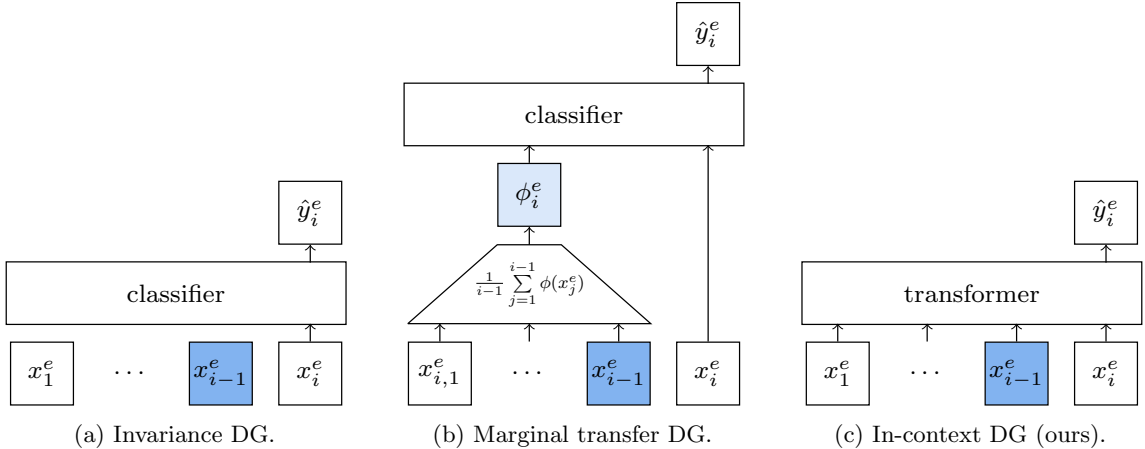

In a parallel strand of research, large language models~\citep[LLMs]{gpt} are taking the world by storm.
LLMs are next-token predictors built with transformers~\citep{attention} and trained on enormous amounts of natural language.
The resulting systems are able to interact with users in the capacity of conversational agents, addressing questions, retrieving facts, summarizing content, drafting emails, and finding bugs in snippets of code.
One impressive feature of LLM systems is their ability to learn \emph{in-context}, that is, to generalize on-the-fly to the eclectic contextual circumstances that users enforce by means of prompting~\citep{brown2020language}.
For example, a good LLM would complete the sequence ``France-Paris Italy-Rome Spain-'' with the sequence ``Madrid,'' effectively learning, from the input itself, that the user is demanding a capital prediction task.
When interacting with large language models, it feels as though we are getting closer to solving the puzzle of out-of-distribution generalization.
Could LLM researchers have found a key piece to this puzzle?

This paper suggests a positive answer, establishing a strong parallel between the concept of \emph{environment} in domain generalization, and the concept of \emph{context} in next-token prediction.
In fact, different environments describe varying contextual circumstances such as time, location, experimental intervention, and other background conditions.
On the one hand, describing \emph{environments as context} opens the door to using powerful next-token predictors off-the-shelf, together with their adaptability to learn in-context, to address domain generalization problems. This allows us to move from coarse domain indices to fine and compositional contextual descriptions, and amortize learning across similar environments.  On the other hand, using \emph{context as environment} can help LLM researchers to use successful domain generalization methods such as distributionally robust optimization ~\citep[DRO]{sagawa2019distributionally,xie2023doremi} across varying contexts.

Based on these insights, we propose a rather natural algorithm, In-Context Risk Minimization (ICRM) as illustrated in \Cref{figure:main:c}.
Given examples $(x^e_i, y^e_i)$ from environment $e$, we propose to address \emph{out}-of-distribution prediction as \emph{in}-distribution next-token prediction, training a machine:
\begin{equation}
  y^e_i \approx h(x^e_i; \underbrace{x^e_1, \ldots, x^e_{i-1}}_{\text{environment} \,\approx\, \text{context}}).
\end{equation}
While the requested prediction $y^e_i$ concerns only the input $x^e_i$, the machine can now pay attention to its test experience so far, as to extract relevant environment information from instance and distributional features.
Our theoretical results show that such in-context learners can amortize context to \emph{zoom-in} on the empirical risk minimizer of the test environment, achieving competitive out-of-distribution performance.
Further, we show that in several settings, the extended input-context feature space in ICRM reveals invariances that ERM-based algorithms ignore. 
Through extensive experiments, we demonstrate the efficacy of \method{} and provide extensive ablations that dissect and deepen our understanding of it.

The rest of the exposition is organized as follows.
\Cref{sec:dg} reviews the fundamentals of domain generalization, centered around the concept of \emph{environment}.
\Cref{sec:icl} explains the basics of next-token prediction, with an emphasis on learning from \emph{context}.
\Cref{sec:main} sews these two threads to propose a framework called \method{} to learn from multiple environments from context, and provides a host of supporting theory.
\Cref{sec:experiments} showcases the efficacy of our ideas in a variety of domain generalization benchmarks, and \Cref{sec:discussion} closes the exposition with some topics for future discussion.

\section{The problem of domain generalization}
\label{sec:dg}

The goal of domain generalization (DG) is to learn a predictor that performs well across a set of domains or environments $\Eall$ \citep{muandet2013domain}.
Environment indices $e \in \Eall$ list different versions of the data collection process---variations that may occur due to time, location, experimental interventions, changes in background conditions, and other contextual circumstances leading to distribution shifts~\citep{irm}.

During training time we have access to a collection of triplets $\mathcal{D} = \{(x_i, y_i, e_i)\}_{i=1}^n$.
Each triplet contains a vector of features $x_i $, a target label $y_i$, and the index of the corresponding training environment $e_i \in \Etr \subset \Eall$.
Formally, each example $(x_i, y_i)$ is sampled independently from a joint probability distribution $P^e(X, Y)$.
Using the dataset $\mathcal{D}$, we set out to learn a predictor $h$ that maps features to labels, while minimizing the worst risk across a set of related but unknown test environments $\Ete = \Eall \setminus \Etr$.
Formally, the standard domain generalization optimization is stated as 
\begin{equation}
  h^{*} = \argmin_h \max_{e \in \Ete} R^e(h),\label{eqn:ood}
\end{equation}
where $R^e(h) = \mathbb{E}_{(X, Y) \sim P^e}[\ell(h(X), Y)]$ is the risk of the predictor $h$ in environment indexed $e$, as measured by the expectation of the loss function $\ell$  with respect to the environment distribution $P^e$.

As one example, we could train a self-driving model $h$ to classify images $x_i$ into a binary label $y_i$ indicating the presence of a pedestrian.
Each training example $(x_i, y_i)$ is hereby collected from $e_i \in \Etr$, one out of the few cities with varying weather conditions from which images are collected.
The goal of~\cref{eqn:ood} is to obtain a predictor that correctly classifies $x$ in new cities $e \in \Ete$ observed during test time.
This has proved to be challenging \citep{lechner2022all}, as predictors trained in different weather conditions exhibited penurious performance in new weather conditions.

Domain generalization is challenging because we do not have access to test environments during training time, rendering \cref{eqn:ood} challenging to estimate.
Therefore, to address the DG problem in practice, researchers have proposed a myriad of algorithms that make different assumptions about the invariances shared between $\Etr$ and $\Ete$.
In broad strokes, domain generalization algorithms fall in one out of the two following categories.
On the one hand, domain generalization algorithms based on invariance~\citep{muandet2013domain, dann, peters2016causal, irm}, illustrated in~\Cref{figure:main:a}, regularize predictors $h(x^e_i)$ as to not contain any information about the environment $e$.
Unfortunately, this results in removing too much signal about the prediction task.
On the other hand, domain generalization algorithms based on marginal transfer~\citep{marginal, adabn, arm, contextvit} extract environment-specific information. These methods implement predictors $h(x^e_i, \phi^e_i)$, where $\phi^e_i = \frac{1}{i-1} \sum_{j=1}^{i-1} \phi(x^e_j)$ is a coarse summary of the environment $e$ in terms of previously observed instances.
Different choices for $\phi$ include kernel functions~\cite[MTL]{marginal}, convolutional neural networks~\citep[ARM]{arm}, and patch embeddings~\citep[Context-ViT]{contextvit}.
Alas, all of these alternatives focus exclusively on distributional features of the environment, diluting relevant ``needle-in-the-haystack'' to be found in individual past examples.
More formally, the size of the representation $\phi$ would have to grow linearly with the size of the training data to describe aspects corresponding to a small group of examples, or non-parametric statistics about their distribution.

As a result, and despite all efforts, no proposal so far convincingly outperforms a simple empirical risk minimization baseline~\citep[ERM]{vapnik} across standard benchmarks~\citep{domainbed, gagnon-audet2023woods, yao2022wild}.
Effectively, ERM simply pools all training data together and seeks the \emph{global} empirical risk minimizer:
\begin{equation}
  h^\dagger = \argmin_h \sum_{e \in \Etr} P(E = e) \cdot R^e(h).
 \label{eqn: erm}
\end{equation}
Is the efficacy of ERM suggesting that environmental information is useless and that the generalization project has reached a stalemate?
We argue that this is not the case.
The key to our answer resides in a recently discovered emergent ability of next-token predictors, namely, in-context learning.

\section{Next-token predictors and in-context learning}
\label{sec:icl}

Next, let us take a few moments to review a seemingly disconnected learning paradigm, next-token prediction.
Here, we are concerned with modeling the conditional distribution
\begin{equation}
  P(Z_{t+1} = z_{t+1} \mid Z_t = z_t, \ldots Z_1 = z_1),\label{eq:next1}
\end{equation}
describing the probability of observing the token $z_{t+1}$ after having previously observed the sequence of tokens $(z_1, \ldots, z_t)$.
The quintessential next-token prediction task is language modeling~\citep{bengio2000neural}, where the sequence of tokens represents a snippet of natural language text.
Language modeling is the workhorse behind the most sophisticated large language models (LLMs) to date, such as GPT-4~\citep{gpt}.
Most LLM implementations estimate~\cref{eq:next1} using a transformer neural network $z_{t+1} \approx h(z_t; z_{t-1}, \ldots, z_1)$~\citep{attention}.

Trained LLMs exhibit a certain ability, termed in-context learning (ICL), quite relevant to our interests.
ICL is the ability to describe and learn about a learning problem from the sequence of a few tokens itself, sometimes called the context or prompt. Many meta-learning methods have been built over the years to impart such an ability to the models \citep{schmidhuber1987evolutionary, finn2017model}.
To illustrate, consider the two following sequences:
\begin{align*}
  (\underbrace{\text{You are talking to a teenager.}}_{c_1} \,\, \underbrace{\text{Write a poem on gravitational fields.}}_{x_1}),\\
  (\underbrace{\text{You are talking to a Physics graduate.}}_{c_2} \,\, \underbrace{\text{Write a poem on  gravitational fields.}}_{x_2}).
\end{align*}
As widely observed, LLMs answer differently to these two sequences, producing two poems, say $y_1$ and $y_2$, each adapted to the assumed audience.
While nothing unexpected is happening here at the sequence level---the model simply produces a high-likelihood continuation to each of the two prompts---we observe a degree of compositional learning, because the LLM can provide different but correct answers to the same question $x_1 = x_2$ when presented under two contexts $c_1$ and $c_2$.
By addressing the very general task of \emph{in}-distribution language modeling, we attain significant \emph{out}-of-distribution abilities in many specific tasks---such as the one of writing poems.

It is a fascinating fact that ICL emerges without supervision.
The training corpus does not contain any explicit division between questions and their context beyond the natural order of the words within each snippet of language in the training data.
However, since we train the machine to produce an enormous amount of completions, many of which start with partially overlapping contexts, the predictor has the opportunity to amortize learning to a significant degree.
While the machine may have never observed the context $\tilde{c}_1 = (\text{You are now \emph{speaking} to a teenager})$, its semantic similarity to $c_1$ above---plus other similar contexts where the word \emph{speaking} appears---endows generalization.
This is the desired ability to generalize over environments described in the previous section, which remained completely out of reach when using coarse domain indices.

\begin{table}
\centering
\begin{tabular}{llll}
  \toprule
  \textbf{paradigm} & \textbf{training data} & \textbf{testing data} & \textbf{estimates}\\
  \midrule
  ERM               & $x, y$    & $x^{e'}$                                             & $P(Y \mid X)$\\
  IRM               & $x, y, e$ & $x^{e'}$                                             & $P(Y \mid \phi^\text{inv}(X))$\\
  LLM               & $z$       & $z_{t}$ \, and context $z_{j < t}$           & $P(Z_{t+1} \mid Z_t, \ldots, Z_1)$\\
  ICRM              & $x, y, e$ & $x^{e'}_t$ and context $c^{e'}_t = (x^{e'}_{j})_{j< t}$ & $P(Y|X,C) \leadsto P^{e'}(Y \mid X)$\\
  \bottomrule
\end{tabular}
\caption{Different learning paradigms discussed in this work, together with their training data and testing data formats, as well as the estimated predictors.
In our \method{}, we amortize the current input $x^{e'}$ and its context $c^{e'}$, containing previously experienced unlabeled examples from the same environment $e'$, and  ``zoom-in'' ($\leadsto$) to the appropriate local risk minimizer.
}
\label{table:estimatetable}
\end{table}

\section{Adaptive domain generalization via in-context learning}
\label{sec:main}

The story has so far laid out two threads.
On the one hand, \Cref{sec:dg} motivated the need for domain generalization algorithms capable of extracting relevant environment-specific features, at both the example and distributional levels.
To this end, we have argued to move beyond coarse environment indices, towards rich and amortizable descriptions.
On the other hand, \Cref{sec:icl} suggests understanding \emph{context} as an opportunity to describe \emph{environments} in precisely this manner.
This section knits these two threads together, enabling us to attack the problem of domain generalization with in-context learners.
The plan is as follows:
\begin{itemize}
\item Collect a dataset of triplets $\mathcal{D} = \{(x_i, y_i, e_i)\}_{i=1}^n$ as described in \Cref{sec:dg}.
Initialize a next-token predictor $\hat{y} = h(x ; c)$, tasked with predicting a target label $y$ associated to the input $x$, as supported by the context $c$.
\item During each iteration of training, select $e \in \Etr$ at random. Draw $t$ examples from this environment at random, construct one input sequence $(x^e_1, \ldots, x^e_t)$ and its associated target sequence $(y^e_1, \ldots, y^e_t)$.
Update the next-token predictor to minimize the auto-regressive loss $\sum_{j=1}^t \ell(h(x^e_j; c^e_j), y^e_j)$, where the context is $c^e_j = (x^e_1, \ldots, x^e_{j-1})$, for all $j = 2, \ldots, t$, and $c^e_1 = \emptyset$.
\item During test time, a sequence of inputs $(x'_1, \ldots, x'_{t'})$ arrives for prediction, one by one, all from the test environment $e' \in \Ete$.
We predict $\hat{y}'_j = h(x'_j, c'_j)$ for $x'_j$, where the context $c'_j = (x'_1, \ldots, x'_{j-1})$, for all $j = 2, \ldots, t'$, and $c'_1 = \emptyset$.
\end{itemize}
We call the resulting method, illustrated in~\Cref{figure:main:c}, \methodfull{} (\method{}).

Next, we develop a sequence of theoretical guarantees to understand the behavior of \method{} in various scenarios.
To orient ourselves around these results, we recall three predictors featured in the exposition so far.
First, the global risk minimizer over the pooled training data, denoted by $h^{\dagger}$ in~\cref{eqn: erm}, estimates $P(Y \mid X)$.
Second, the environment risk minimizer, denoted by $h^e(x)$ for environment $e$, estimates $P(Y \mid X, E)$.
Third, our in-context risk minimizer estimates the conditional expectation $P(Y \mid X, C)$, denoted by
\begin{equation}
        \tilde{h} = \argmin_{h} \sum_{j=1}^{t}\mathbb{E}_{(X, C, Y)}[\ell(h(X_j; C_j), Y_j)].
        \label{eqn: icl}
\end{equation}
The sequel focuses on the binary cross-entropy loss $\ell$. Our first result shows that, in the absence of context, \method{} \emph{zooms-out} to behave conservatively. 

\begin{restatable}[Zoom-out]{proposition}{zoomout}
\label{prop:zoomout}
    In the absence of context, \method{} behaves as the global empirical risk minimizer across the support of the training environments, i.e., $\tilde{h}(\cdot \; ;\; \emptyset) = h^{\dagger}(\cdot)$. 
\end{restatable}

Having established the connection between ICRM and ERM in the absence of any context, we now study the benefits of ICRM in the presence of sufficiently long contexts. The following result shows that, when provided with context from a training environment $e \in \Etr$, our \method{} \emph{zooms-in} and behaves like the appropriate environment risk minimizer, as shown in~\Cref{table:estimatetable}. In the next result, we assume that $P(Y = 1 \mid X=x, E=e)$ is parametrized and described by a function $h^\star(x, \theta^e_x)$, where $\theta_x^e$ describes features of the environment that are relevant to the query $x$, for all $e \in \Eall$. We assume an ideal \emph{amortization function} $b$ that takes the query $x$ and context $C_t$ as input and approximates $\theta_x^e$ and the sequence of random variables $b(X,C_t)$ converges almost surely to $\theta_X^E$.

\begin{restatable}[Full iid zoom-in]{theorem}{zoomin}\label{thm: zoomin}
    Let $h^{\star}(x, \theta^e_x)$ describe $P(Y = 1 \mid X=x, E=e)$ for all $e \in \Eall$.
  Furthermore, we assume the existence of an \emph{amortization function} $b(X, C_t) \stackrel{a.s.}{\rightarrow} \theta^{E}_X$.
  Then, \method{} zooms-in on the environment risk minimizer by achieving a cross-entropy loss
  \begin{equation*}
   \lim_{t\rightarrow\infty} H(Y \mid X,C_t) = H(Y \mid X,E).
  \end{equation*}
  Further, if $I(Y;E \mid X)>0$, \method{} has better performance than the global risk minimizer.
\end{restatable}

In the previous result, we established that ICRM converges to empirical risk minimizer of the environment under infinitely long contexts. Next, we show that \method{} can partially zoom-in on the appropriate environment risk minimizer even with contexts of length of one.

\begin{restatable}[Partial iid zoom-in]{theorem}{partialzoomin}\label{thm: partialzoomin}
    Suppose the joint distribution $((X_1, \cdots X_t), (Y_1, \ldots, Y_t) ,E)$ is Markov with respect to a Bayesian network, each query $X$ and environment $E$ are statistically dependent and form the Markov blanket of $Y$.    Then, \method{} partially zooms-in on the environment risk minimizer, improving the performance of the global risk minimizer in terms of the cross-entropy loss. Further, the improvement is strictly monotonic in context length $t$.
\end{restatable}

Next, we move to the out-of-distribution setting where the test environments can be quite different from the train environments. 
To provide theory for a domain generalization result, we must place some assumptions on the data generation process.
In particular, and for all $e\in \Eall$, let
\begin{equation}
z \mid y,e \sim \mathcal{N}(\mu_e^y, \Sigma_e^y), \text{ and } x \leftarrow g(z),
\label{eqn: dgp_ood}
\end{equation}
where the latent variables $z$ are sampled conditional on the label $y$ and environment $e$ from a Gaussian distribution with mean and covariance depending on $(y,e)$, and are then mixed by a map $g$ to generate the observations $x$. We summarize the environment in terms of the parameter vector $\gamma_e = \big[(p_e^{y}, \mu_e^{y}, \Sigma_e^{y})_{y\in \{0,1\}}\big]$, where $p_e^{y}$ is the probability of label $y$ in environment $e$.
Our next result shows that ICL algorithms that learn $h(x;c)$ exhibit robust behavior under distribution shifts. 
In contrast, standard predictors $h(x)$ can fail on novel environments from~\cref{eqn: dgp_ood}.

 Define $\delta_e$ to be a permutation of $\gamma_e$ that swaps the two components.  We construct the Voronoi cells corresponding to the points in the union of sets $\{\gamma_e\}_{e\in \mathcal{E}_{tr}}$ and $\{\delta_e\}_{e\in \mathcal{E}_{tr}}$. The set of points in the Voronoi cell corresponding to  $\{\gamma_e\}_{e\in \mathcal{E}_{tr}}$ define the \emph{Voronoi cell of the training environments}. Next, we show that  ICL can perform in novel test environments sufficiently far away from the training environments, so long as they are in the Voronoi cells of training environments.

\begin{restatable}[Full ood zoom-in]{theorem}{oodzoomin}\label{thm: oodzoomin}
    Consider data triplets $(x, y, e)$ generated from $z \sim \mathcal{N}(\mu_{e}^{y}, \Sigma_{e}^{y})$ and $x \leftarrow g(z)$, for all environments $e \in \Eall$, where $g$ is the identity map (see \cref{sec: thms_proofs} for extensions to general diffeomorphisms $g$). There exists an ICL algorithm that produces Bayes optimal predictions for all the test environments that fall in the Voronoi cells of the training environments.
\end{restatable}

\section{\method{} under the lens of invariance}

Common advice in domain generalization recommends following the \emph{invariance principle} to learn robust predictors~\citep{peters2016causal, irm}.
One simple version of the invariance principle is to ``select those inputs leading to stable regression coefficients across training environments.''
At first sight, one could argue that the proposed \method{} does not adhere to such an invariance principle, as it is adapting to environment-specific information provided in the form of context.
However, as some examples can show, \method{}'s implementation of ERM on the extended input-context feature space reveals invariant predictors that a vanilla implementation of ERM on the standard feature space fails to find.
To see this, consider a linear least-squares regression problem mapping two inputs $(x_1, x_2)$ into a target $y$ under multiple training environments $e \in \Etr$ as:
\begin{equation}
 y  = \alpha \cdot x_1 + \beta \cdot \mu_2^e + \varepsilon,\label{eq:toy}
\end{equation}
where $\mu_i^e = \mathbb{E}[x_i \mid E=e]$, the pair $(\alpha, \beta)$ are invariant regression coefficients, and $\varepsilon$ is an independent noise term.
Algorithmically, we make one simplifying assumption for pedagogic purposes.
In particular, during training, we provide ICRM directly with the relevant extended feature space $(x_1, x_2, \mu^e_1, \mu^e_2)$, instead of requiring the algorithm to learn such representation from general-form sequential context.

In this setup, ICRM learns to predict using $\alpha \cdot x_1 + 0 \cdot x_2 + 0 \cdot \mu_1^e + \beta \cdot \mu_2^e$.
In contrast, ERM trains a linear model on $(x_1, x_2)$, learning to predict using $\tilde{\alpha} \cdot x_1 + \tilde{\beta} \cdot x_2$.
This is the main point:
if $\beta \neq 0$ and $\text{cov}(x_1, x_2) \neq 0$, ERM's estimate of the invariant coefficient is biased, $\tilde{\alpha} \not= \alpha$, and as a result the error of ERM in a new environment grows with the variance of $x_1$. On the other hand, ICRM estimates the invariant coefficient for $x_1$ perfectly and the error that it experiences is independent of variance of $x_1$ regardless of the context seen so far. As a result, the error of ERM is guaranteed to be worse than ICRM provided the variance of $x_1$ is sufficiently large. For a formal derivation and generalization of these claims, see~\Cref{sec: thms_proofs}. In our experiments too, we observe that ICRM is able to generalize \emph{zero-shot} to novel test environments.

We believe that ICRM, and more generally ICL, provide one interesting new viewpoint on invariance.
On the one hand, prior DG algorithms advocated to remove features as a guide to reveal invariance.
On the other hand, in-context learners suggest that extending features with context affords invariance otherwise unnoticed.
This needs further clarification: while the process of zooming-in to an environment risk minimizer does not provide us with an invariant predictor over the original feature space, the \emph{process of zooming-in} is in many cases an invariant mechanism over the extended feature space.
These points are reminiscent of the discussion about ``fragility'' in the philosophy of causation~\citep{sep-causation-counterfactual}.
Does smoking cause cancer?
Not invariably, at least not across all contexts or environments.
However, smoking does cause cancer invariably---across all contexts or environments---when extending the feature space as to include additional causes such as diet, genetic predispositions, and the number of smoked cigarettes.
The ever-growing collection of causes approaches what John Stuart Mill called the \emph{total cause}.
Then, learning across a diverse set of environments should allow the machine to pay attention to those that matter for robust prediction.
In short, we afford invariance at the expense of constraining the diameter of the environment.
In the extreme, when constraining the environment to contain only one smoker, we can always find an invariant predictor.

\section{Experiments}  \label{sec:experiments}

To evaluate the efficacy of our \method, the following subsections are empirical investigations to answer the following questions, respectively:
\begin{enumerate}
    \item How does \method{} fare against competitive DG algorithms, for different context sizes?
    \item How does the \method{} perform in the absence of domain labels?
    \item What is the impact of model architecture on \method{}'s gains?
    \item Can \method{} search for query relevant ``needle-in-the-haystack'' signals?
\end{enumerate}

\begin{table}[htb!]
    \caption{Average/worst ood test accuracy for various counts of context samples for Adaptive Risk Minimization (ARM), Empirical Risk Minimization (ERM), Test Entropy Minimization (TENT) and our \method{} on FEMNIST, Rotated MNIST, WILDS Camelyon17 and  Tiny-ImageNet-C.}
    \begin{center}
    \begin{tabular}{lccccc|cccccc}
        \toprule
        \textbf{Data / method} & \multicolumn{5}{c}{\textbf{Average test accuracy}} & \multicolumn{5}{c}{\textbf{Worst case test accuracy}} \\
        \coloredMidrule{white}{alternateRowColor}
        \rowcolor{alternateRowColor}
        FEMNIST & 0 & 25 & 50 & 75 & 100 & 0 & 25 & 50 & 75 & 100  \\
        \coloredBelowRuleSep{white}
        \quad ARM       &  49.5 &  83.9 &  84.4 & 84.7 & 84.6  & 23.6  &  59.5 & 60.7  & 57.0  & 58.8 \\
        \quad TENT  & 78.1 & 77.9 & 81.2 & 82.5 & 83.3 & 55.2 & 57.2 & 63.3 & 65.9 & 67.2 \\
        \quad ERM       &  \textbf{79.3}  & 79.3  & 79.3  & 79.3 & 79.3 &  59.0  & 59.0 & 59.0 & 59.0 &  59.0\\
        \quad ICRM       &  78.7  & \textbf{87.2} & \textbf{87.4}  & \textbf{87.5} &  \textbf{87.8} &  \textbf{59.8} & \textbf{69.3}  & \textbf{70.6}  & \textbf{70.6} & \textbf{70.6} \\
        \coloredMidrule{white}{alternateRowColor}
        \rowcolor{alternateRowColor}
        Rotated MNIST & 0 & 25 & 50 & 75 & 100 & 0 & 25 & 50 & 75 & 100  \\
        \coloredBelowRuleSep{white}
        \quad ARM       &  36.5 & 94.2  &  95.1 & 95.3 & 95.5 &  28.2 & 85.3  & 87.2  & 87.9 & 87.9 \\
        \quad TENT & 94.1 & 88.0 & 91.9 & 93.8 & 94.3 & 80.2 & 88.5 & 88.5 & 80.2 & 81.3\\
        \quad ERM       & \textbf{94.2} & 94.2 & 94.2  & 94.2 & 94.2 &  80.8 & 80.8  & 80.8  & 80.8 & 80.8  \\
        \quad ICRM       &  93.6 &   \textbf{96.1} &  \textbf{96.2} & \textbf{96.2}& \textbf{96.2} &  \textbf{82.5} &  \textbf{88.5} &  \textbf{88.5} & \textbf{88.8} & \textbf{88.8} \\
        \coloredMidrule{white}{alternateRowColor}
        \rowcolor{alternateRowColor}
        WILDS Camelyon17 & 0 & 25 & 50 & 75 & 100 & 0 & 25 & 50 & 75 & 100  \\
        \coloredBelowRuleSep{white}
        \quad ARM       &  61.2 &  59.5 & 59.7  & 59.7 & 59.7 & \multicolumn{5}{c}{\multirow{4}{*}{same as average accuracy}} \\
        \quad TENT       & 67.9 & 81.8 & 87.2 & 89.4 & 89.4 &   \\
        \quad ERM       &  68.6  & 68.6  &  68.6 & 68.6 & 68.6 &   \\
        \quad ICRM       &  \textbf{92.0} & \textbf{90.7}  & \textbf{90.8}  & \textbf{90.8} & \textbf{90.8} &   &   &   &  &  \\
        \coloredMidrule{white}{alternateRowColor}
        \rowcolor{alternateRowColor}
        Tiny ImageNet-C & 0 & 25 & 50 & 75 & 100 & 0 & 25 & 50 & 75 & 100  \\
        \coloredBelowRuleSep{white}
        \quad ARM       &  30.8 &  31.0 &  31.0 &  31.0 & 31.0 & 8.2  & 8.3  &  8.2 & 8.3 & 8.2 \\
        \quad TENT       & 31.7 & 1.6 & 1.7 & 2.0 & 2.1 & 9.4 & 1.2  & 1.4 & 1.6 & 1.6\\
        \quad ERM       &  31.8 &  31.8 & 31.8  & 31.8 & 31.8 & 9.5  &  9.5 &  9.5 &  9.5 & 9.5  \\
        \quad ICRM       &  \textbf{38.3} & \textbf{39.2}  & \textbf{39.2}  & \textbf{39.2} & \textbf{39.2} & \textbf{18.8}  &  \textbf{19.2} & \textbf{19.5}  & \textbf{19.5} & \textbf{19.4} \\
        \bottomrule
    \end{tabular}
    \end{center}
       \label{table:main_sota}
\end{table}

In the following experiments, we compare \method{} against several prior methods designed to address domain generalization. Key competitors to our approach are marginal transfer based algorithms, which summarize observed inputs in each environment as a coarse embedding, as described in \Cref{sec:dg}. Among these methods, we compare with Adaptive Risk Minimization~\citep[ARM]{arm} and TENT~\citep{wang2020tent}. As a strong baseline, we also include ERM in our experimental protocol. 
To ensure a fair comparison across different algorithms for each dataset, we use a standardized neural network backbone (ConvNet or ResNet-50 depending on the dataset) as described in ~\Cref{sec: experimental setup}. For \method{}, the same backbone is used to featurize the input, which is then processed by the decoder-only Transformer \citep{attention} architecture from the GPT-2 Transformer family \citep{radford2019language}. For fair comparisons, we adhere to DomainBed's protocols for training, hyperparameter tuning, and testing~\citep{domainbed}. We describe our experimental setup in detail in~\Cref{sec: experimental setup}

We assess these methods across four image classification benchmarks, each offering a unique problem setting. FEMNIST~\citep{cohen2017emnist} contains MNIST digits and handwritten letters from individual writers as environments. Rotated MNIST concerns varied rotational angles as environments. Tiny ImageNet-C~\citep{hendrycks2019benchmarking} introduces diverse image corruptions to create multiple environments. WILDS Camelyon17~\citep{koh2021wilds} studies tumor detection and sourcing data from multiple hospitals as distinct environments. More details are provided in ~\Cref{sec:datasets}.

\subsection{Adaptation to distribution shift}
To study the adaptation of various approaches to distribution shifts, for each dataset and algorithm, we report performance across varying counts of context samples from the test environment, specifically at 0, 25, 50, 75, and 100 samples. We report an average across three independent runs of the entire sweep and its corresponding standard
error, where we select the model with hyperparameters corresponding to the highest validation accuracy.
\Cref{table:main_sota} summarizes the results of our experiments.
\method{} consistently outperforms all methods across different numbers of in-context test samples except at 0 context on FEMNIST and Rotated MNIST, where ERM marginally exceeds by 1\%. Further, these gains persist over both the worst group and average accuracy across testing environments.
\Cref{fig: mainplots} zooms into the model's performance between no-context and 25 context samples, highlighting the consistent superiority of \method{} even with a few in-context samples. Additionally, \method{}  demonstrates gains in performance even in the absence of test context. Specifically for both WILDS Camelyon17 and Tiny ImageNet-C,
\method{} achieves significantly superior performance than other baselines during inference without leveraging any context from the test environment. 
The training regimen of \method{} enables the model to identify contextual images relevant to the current query, resulting in a better featurizer compared to standard ERM, which is limited to updating based on the current input alone.
In \Cref{sec: secattnmap}, we present instances of such selections identified by \method{} for a given query in the context.

\subsection{Robustness of \method{} in the absence of environment labels}
As outlined in \Cref{sec:main}, the training regimen of \method{} assumes a dataset $\mathcal{D} = \{(x_i, y_i, e_i)\}_{i=1}^n$ collected under multiple training environments $e_i \in \Etr$. However, in scenarios lacking such domain separation during training, does \method{} continue to show an edge over ERM baselines? To study this question, we modify the sampling strategy: rather than constructing context vectors containing examples from one environment, we construct context vectors containing iid samples from all of the environments pooled together. To continue to test for out-of-distribution generalization, however, we evaluate the performance on examples from a novel test environment. We term this modified approach ICRM-Mix.

\begin{table}[htb!]
    \caption{Average/worst ood test accuracies for \method{} and ICRM-Mix across FEMNIST, Rotated MNIST, WILDS Camelyon17 and  Tiny-ImageNet-C. ICRM-Mix trains on sequences with samples drawn i.i.d. from the unified dataset comprising various environments.}
    \begin{center}
    \begin{tabular}{lccccc|cccccc}
        \toprule
        \textbf{Data / method} & \multicolumn{5}{c}{\textbf{Average test accuracy}} & \multicolumn{5}{c}{\textbf{Worst case test accuracy}}\\
        \coloredMidrule{white}{alternateRowColor}
        \rowcolor{alternateRowColor}
        FEMNIST & 0 & 25 & 50 & 75 & 100 & 0 & 25 & 50 & 75 & 100  \\
        \coloredBelowRuleSep{white}
        
        \quad ICRM       &  78.7  & 87.2 & 87.4  & 87.5 &  87.8 &  59.8 & 69.3  & 70.6  & 70.6 & 70.6 \\
        \quad ICRM-Mix  & 77.6 & 81.1 & 81.1 & 80.9 &  80.9 & 57.5 & 62.7 & 65.0 & 64.1 & 62.9 \\
        \coloredMidrule{white}{alternateRowColor}
        \rowcolor{alternateRowColor}
        Rotated MNIST & 0 & 25 & 50 & 75 & 100 & 0 & 25 & 50 & 75 & 100  \\
        \coloredBelowRuleSep{white}
        \quad ICRM       &  93.6 &   96.1 &  96.2 & 96.2& 96.2 &  82.5 &  88.5 &  88.5 & 88.8 & 88.8 \\
        \quad ICRM-Mix   &  88.9 & 92.6 & 92.7 & 92.6 & 92.7 & 68.8 & 77.1 & 76.8 & 76.4 & 76.6  \\
        \coloredMidrule{white}{alternateRowColor}
        \rowcolor{alternateRowColor}
        WILDS Camelyon17 & 0 & 25 & 50 & 75 & 100 & 0 & 25 & 50 & 75 & 100  \\
        \coloredBelowRuleSep{white}
        \quad ICRM       &  92.0 & 90.7  & 90.8  & 90.8 & 90.8 &  \multicolumn{5}{c}{\multirow{2}{*}{same as average accuracy}} \\
        \quad ICRM-Mix   & 92.9 & 90.7 & 90.8  & 90.7 & 90.7 \\
        \coloredMidrule{white}{alternateRowColor}
        \rowcolor{alternateRowColor}
        Tiny ImageNet-C & 0 & 25 & 50 & 75 & 100 & 0 & 25 & 50 & 75 & 100  \\
        \coloredBelowRuleSep{white}
        \quad ICRM       & 38.3 & 39.2  & 39.2  & 39.2 & 39.2 & 18.8  &  19.2 & 19.5  & 19.5 & 19.4 \\
        \quad ICRM-Mix   & 38.4 & 39.3  & 39.3 & 39.3 & 39.3 & 18.7 & 19.2 & 19.4 & 19.5 & 19.4 \\
        \bottomrule
    \end{tabular}
    \end{center}
    \label{table:ablation_iid_icl}
\end{table}

\Cref{table:ablation_iid_icl} contrasts the performance of \method{} with ICRM-Mix.
\method{} consistently outperforms ICRM-Mix across varying counts of in-context samples on both FEMNIST and Rotated MNIST. Surprisingly, ICRM-Mix and \method{} perform similarly on WILDS Camelyon17 and Tiny ImageNet-C. Consider a setting where the model benefits the most attending to examples from the same class or related classes. If classes are distributed uniformly across domains, then ICRM and ICRM-mix are bound to perform similarly. Consider another setting where the model benefits the most by attending to environment specific examples such as characters drawn by the same user. In such a case, ICRM and ICRM-mix have very different performances.

\subsection{Understanding the impact of architecture}

To dissect the performance gains potentially arising from \method{}'s transformer architecture, we explore two additional competitors.
On the one hand, we train an ERM baseline, ERM$^+$ using an identical architecture to IRCM, but without context.
On the other hand, we train an ARM baseline, ARM$^+$, where the input-context pair is provided to the same transformer as the one used by ICRM.
This is in contrast to the original implementation of ARM, where input and context are concatenated together along the channel dimension, and sent to classification to a convolutional neural network.

\begin{table}[htb!]
    \caption{Average out-of-distribution test accuracies for ARM$^{+}$ and ERM$^{+}$ in contrast to their base algorithms, ARM and ERM across FEMNIST, Rotated MNIST, WILDS Camelyon17 and  Tiny-ImageNet-C. }
    \begin{center}
    \begin{tabular}{lccccc|cccccc}
        \toprule
        \textbf{Data / method} & \multicolumn{5}{c}{\textbf{Average test accuracy}} & \multicolumn{5}{c}{\textbf{Worst case test accuracy}}\\
        \coloredMidrule{white}{alternateRowColor}
        \rowcolor{alternateRowColor}
        FEMNIST & 0 & 25 & 50 & 75 & 100 & 0 & 25 & 50 & 75 & 100  \\
        \coloredBelowRuleSep{white}
        \quad ARM       &  49.5 &  83.9 &  84.4 & 84.7 & 84.6  & 23.6  &  59.5 & 60.7  & 57.0  & 58.8 \\
        \quad ARM$^{+}$ & 71.4 & 83.4 &  84.0 & 83.8 & 83.5 & 51.7 & 63.0 & 64.0 & 60.7 & 62.0 \\
        \midrule
        \quad ERM  &  79.3  & 79.3  & 79.3  & 79.3 & 79.3 &  59.0  & 59.0 & 59.0 & 59.0 &  59.0\\
        \quad ERM$^{+}$ & 77.4 & 77.4 & 77.4 & 77.4 & 77.4 & 53.3 & 53.3  & 53.3 & 53.3 & 53.3 \\
        \coloredMidrule{white}{alternateRowColor}
        \rowcolor{alternateRowColor}
        Rotated MNIST & 0 & 25 & 50 & 75 & 100 & 0 & 25 & 50 & 75 & 100  \\
        \coloredBelowRuleSep{white}
        \quad ARM       &  36.5 & 94.2  &  95.1 & 95.3 & 95.5 &  28.2 & 85.3  & 87.2  & 87.9 & 87.9 \\
        \quad ARM$^{+}$   & 86.9 & 92.6 & 92.7 & 92.8 & 92.8  & 71.4 &  80.9 & 81.0  & 81.2 & 81.1\\
        \midrule
        \quad ERM       & 94.2 & 94.2 & 94.2  & 94.2 & 94.2 &  80.8 & 80.8  & 80.8  & 80.8 & 80.8  \\
        \quad ERM$^{+}$ & 94.3 & 94.3 & 94.3  & 94.3 & 94.3 &  81.9 & 81.9  & 81.9  & 81.9 & 81.9  \\
        \coloredMidrule{white}{alternateRowColor}
        \rowcolor{alternateRowColor}
        WILDS Camelyon17 & 0 & 25 & 50 & 75 & 100 & 0 & 25 & 50 & 75 & 100  \\
        \coloredBelowRuleSep{white}
         \quad ARM       &  61.2 &  59.5 & 59.7  & 59.7 & 59.7 &   \multicolumn{5}{c}{\multirow{2}{*}{same as average accuracy}} \\
        \quad ARM$^{+}$ & 55.8 & 55.1 & 55.0 & 55.0 & 55.0 &  &  &  &  &  \\
        \midrule
        \quad ERM       &  68.6  & 68.6  &  68.6 & 68.6 & 68.6 &  \multicolumn{5}{c}{\multirow{2}{*}{same as average accuracy}} \\
        \quad ERM$^{+}$ & 50.1 & 50.1 & 50.1 & 50.1 & 50.1 &  &  &  &  & \\
        \coloredMidrule{white}{alternateRowColor}
        \rowcolor{alternateRowColor}
        Tiny ImageNet-C & 0 & 25 & 50 & 75 & 100 & 0 & 25 & 50 & 75 & 100  \\
        \coloredBelowRuleSep{white}
        \quad ARM       &  30.8 &  31.0 &  31.0 &  31.0 & 31.0 & 8.2  & 8.3  &  8.2 & 8.3 & 8.2 \\
        \quad ARM$^{+}$ & 5.5 &5.7  & 5.7 & 5.7 & 5.7 & 1.9 & 1.9 & 1.9 & 1.9 & 1.9  \\
        \midrule
        \quad ERM       &  31.8 &  31.8 & 31.8  & 31.8 & 31.8 & 9.5  &  9.5 &  9.5 &  9.5 & 9.5  \\
        \quad ERM$^{+}$ & 29.7 & 29.7 & 29.7 & 29.7 & 29.7 & 8.3 & 8.3 & 8.3 & 8.3 & 8.3 \\
        \bottomrule
    \end{tabular}
    \end{center}
    \label{table:ablation_pluses}
\end{table}

\Cref{table:ablation_pluses} presents the performance of both ERM$^{+}$ and ARM$^{+}$ relative to their base models, ERM and ARM, across four benchmark datasets. ARM$^{+}$ demonstrates superior zero-shot performance over ARM on both FEMNIST and Rotated MNIST. However, ARM maintains a performance advantage over ARM$^{+}$ across varying counts of in-context samples on WILDS Camelyon17 and Tiny ImageNet-C, with a notably pronounced difference on the latter.  Similarly, ERM either matches or outperforms ERM$^{+}$ on all four datasets.

\begin{figure}[!htb]
  \centering
  \includegraphics[width=\textwidth]{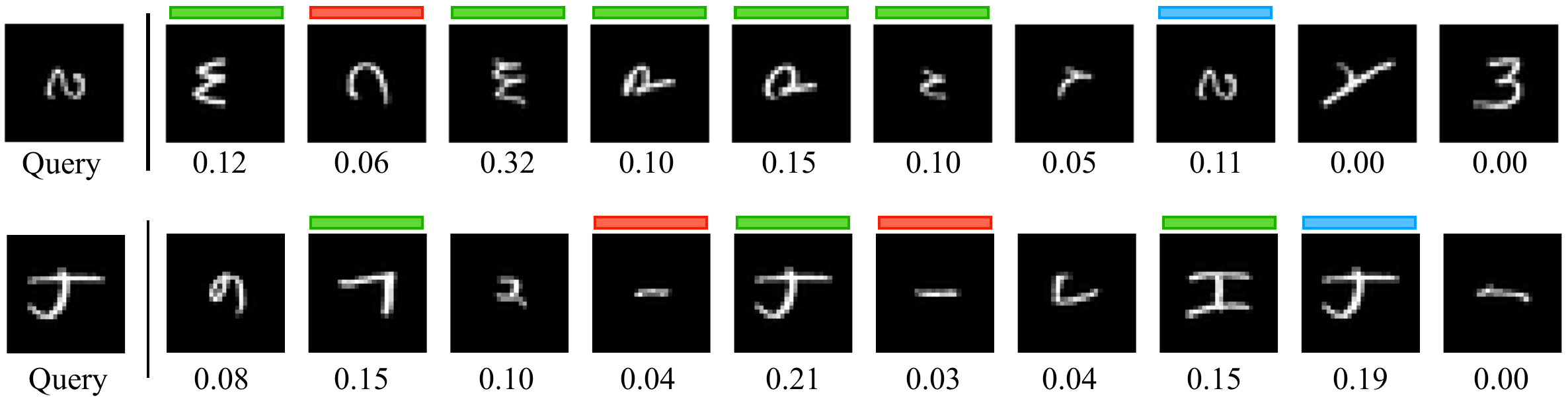}
  \includegraphics[width=\textwidth]{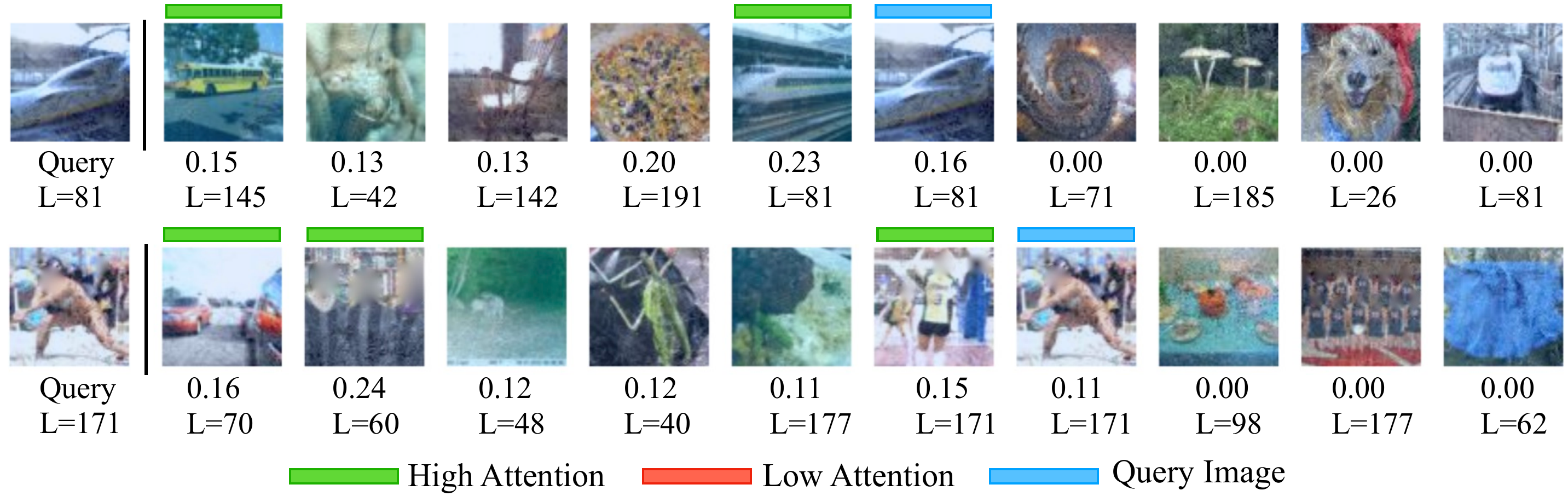}
  \caption{Attention scores for query images within randomized input sequences from test environments, as determined by \method{} on FEMNIST (top two rows) and Tiny ImageNet-C (bottom two rows). `L' denotes the label of a given image.}
  \label{fig:attn_maps}
\end{figure}

\subsection{Investigating attention in \method{}}\label{sec: secattnmap}
As discussed in \Cref{sec:dg}, one special feature of \method{} is its ability to learn an amortization function by paying attention to the input query and its context. To better understand this nuanced functionality, we turn our focus to visualizing attention maps of a trained \method{} model. Specifically, we construct a random sequence of data from the test environment and examine the attention scores between each example in this sequence and a novel input query across different heads.
\Cref{fig:attn_maps} illustrates attention scores from a single head for two query images (marked in blue) for FEMNIST and Tiny ImageNet-C.
The top row reveal that the model selectively attends to images featuring at least two curved arcs (marked in green) while paying little attention to a partial circle (highlighted in red). Additionally, when the query image is interpreted as a 90-degree clockwise rotated digit ``2'', the model extends its attention to other augmentations of ``2'' within the prompt.
Remarkably, such attention patterns emerge on unseen domains using only unlabeled examples from them, underscoring the potential of amortization!
Similarly, in the second row, attention is predominantly allocated to lines of length similar to that of the query (also in green), thereby largely disregarding shorter lines (shown in red).
The last two rows in~\Cref{fig:attn_maps} show that the model, when presented with a query image of a train, attends not only on other trains but also on a bus---indicating a semantic understanding of similarity. In the bottom panel, the model demonstrates a capability to discern individuals across samples within the prompt.

\section{Discussion}
\label{sec:discussion}

We have introduced \methodfull{} (\method{}), a framework to address domain generalization as next-token prediction.
\method{} learns in-context about environmental features by paying attention to unlabeled instances as they arrive.
In such a away, \method{} dynamically zooms-in on the test environment risk minimizer, achieving competitive out-of-distribution generalization.

\method{} provides a new perspective on invariance.
While prior work on DG focused on information removal as a guide to generalization, ICRM suggests that extending the feature space with the relevant environment information affords further invariance.
By addressing the very general problem of next-token prediction \emph{in}-distribution, we amortize the performance over many specific \emph{out-of}-distribution tasks.
This happens by virtue of moving beyond coarse environment indices, into rich, hierarchical, and partially-overlapping context vectors. 
More generally, by framing DG in terms of next-token prediciton, we enable learning machines to fully exploit data in \emph{natural order}, more closely mimicking the human learning experience.
As L\'eon Bottou once said, \emph{Nature does not shuffle data}.
As a word of caution, we must conduct research to guarantee that in-context learners do not ``zoom-in'' on toxic spurious correlations with high predictive power in certain environments. 

We would like to close with a quote from \citet{andersen2022predictive}, who claim that the central property of zooming-in on the relevant information
\begin{quote}
\small
refers to a cognitive agent’s ability to intelligently ignore irrelevant information and zero in on those aspects of the world that are relevant to their goals. The relevance realization framework suggests that the brain achieves this feat by attempting to balance the competing goals of remaining efficient in the current environment while also being resilient in the face of environmental perturbations.
\end{quote}
Paralleling the examples from \citet{andersen2022predictive}, we are excited to work to further understand how next-token prediction and in-context learning serves as a powerful mechanism to amortize and dynamically navigate trade-offs such as such as the efficiency-resiliency tradeoff, the exploration-exploitation tradeoff, specialization-generalization tradeoff, and focusing-diversifying tradeoff.

\clearpage
\newpage
\section*{Acknowledgements}
We are thankful to Martin Arjovsky, L\'eon Bottou, Elvis Dohmatob, Badr Youbi Idrissi, Maxime Oquab, and Ahmed Touati for their valuable feedback and help. 

\bibliographystyle{plainnat}
\bibliography{bibliography}

\clearpage
\newpage

\appendix
\input{appendix}

\end{document}

%% file: math_commands.tex
\usepackage{amsmath,amsfonts,bm}

\def\eqref#1{equation~\ref{#1}}

\def\1{\bm{1}}

\DeclareMathAlphabet{\mathsfit}{\encodingdefault}{\sfdefault}{m}{sl}
\SetMathAlphabet{\mathsfit}{bold}{\encodingdefault}{\sfdefault}{bx}{n}

\DeclareMathOperator*{\argmin}{arg\,min}

%% file: appendix.tex
\startcontents[appendices]
\printcontents[appendices]{}{1}{\section*{Appendix}}

\section{Theorems and Proofs}
\label{sec: thms_proofs}

\subsection{Proof of Proposition \ref{prop:zoomout}}

\begin{lemma}\label{prop: bayes_opt} \textbf{\method{} is Bayes optimal at all context lengths.}  Suppose $\ell$ is the cross-entropy loss and the labels $Y$ are binary. The optimal in-context learner $\tilde{h}$ (\eqref{eqn: icl}) satisfies the following condition, i.e., for each $k \in [t] $ 
\begin{equation}
    \tilde{h}(x_k;c_k) = \mathbb{E}[Y|X_k=x_k, C_k=c_k],
\end{equation}
for almost all  $(c_k,x_k)$ in the support of training distribution except over a set of a measure zero, and where the expectation is over $Y$ conditional on $[c_k,x_k]$. In other words, the in-context learner is Bayes optimal at each context length.
\end{lemma}

\begin{proof}

In this result, we consider the problem of binary classification. Suppose $h(x_k; c_k)$ is the probability of class $Y=1$. Define $\bar{h}(x_k; c_k) = \big[h(x_k; c_k), 1-h(x_k;c_k)\big]$ describing the probability of both the classes.

From \eqref{eqn: icl}, recall that  the objective of \method{} is to minimize
\begin{equation}
    \begin{split}
     \sum_{j=1}^{t}\mathbb{E}_{(X, C, Y)}[\ell(h(X_j; C_j), Y_j)]
    \end{split}
    \label{eqn: icl_objective}
\end{equation}

Consider one of the terms in the sum above -  $\mathbb{E}\big[\ell(h(X_k; C_k),Y_k)\big]$.  Substituting $\ell$ as the cross-entropy in this term, we obtain 
$$\mathbb{E}\big[\ell(h(X_k; C_k),Y_k)\big] = H(Y_k |X_k, C_k) + \mathbb{E}\big[\mathsf{KL}\big(P(Y_k|X_k, C_k) \big\| \bar{h}(X_k; C_k)\big)\big]$$

If $\bar{h}(X_k; C_k) = P(Y_k|X_k, C_k)$, then the second term in the above is zero and $\mathbb{E}\big[\ell(h(X_k; C_k),Y_k)\big]$ equals $ H(Y_k |X_k; C_k)$. Since KL divergence is always non-negative,  $H(Y_k |X_k, C_k)$ corresponds to the lowest value that can be achieved by $\mathbb{E}\big[\ell(h(X_k; C_k),Y_k)\big]$.  
If $\bar{h}(X_k; C_k) = P(Y_k|X_k, C_k)$ for all $k \in [t]$, then each of the terms in the sum in \eqref{eqn: icl_objective} are minimized. As a result, $\bar{h}(X_k; C_k) = P(Y_k|X_k, C_k)$ for all $k \in [t]$ is a solution to \eqref{eqn: icl}.  

Consider another minimizer $h^{'}$ of \eqref{eqn: icl} and define the  corresponding distribution $\bar{h}^{'}$. For each $k \in [t]$,  the second term  $\mathbb{E}\big[\mathsf{KL}(P(Y_k|X_k, C_k) \| \bar{h}^{'}(X_k; C_k)\big]$ has to be zero for $\bar{h}^{'}$ to be a minimizer. 

 If $\mathbb{E}\big[\mathsf{KL}(P(Y_k|X_k, C_k) \| \bar{h}^{'}(X_k; C_k)\big]=0$, then we claim that $\bar{h}^{'}(x_k; c_k) = P(Y_k|X_k=x_k, C_k=c_k)$ for almost all $(x_k, c_k)$ in the support of training distribution except over a set of measure zero. If the probability measure associated with $X_k, C_k$ is absolutely continuous w.r.t Lebesgue measure, then this follows from Theorem 1.6.6 \citep{ash2000probability}. If the probability measure associated with $X_k, C_k$ is absolutely continuous w.r.t counting measure, then this trivially follows. 
\end{proof}

We proved the above result for classification and cross-entropy loss for measures over $X,C$ that are either absolutely continuous w.r.t Lebesgue measure or the counting measure. It is easy to extend the above result for regressions and least square loss; see Lemma 1 in \cite{ahuja2023closer}.

\zoomout*

\begin{proof}
From~\Cref{prop: bayes_opt}, it follows that $ \tilde{h}(x_k; c_k) = \mathbb{E}[Y|X_k=x_k, C_k=c_k]$. The solution to empirical risk minimization $h^{\dagger}(x) = \mathbb{E}[Y|X_1=x]$, where the expectation is  computed over the training distribution of $Y$ conditional on $x$. When the context is empty, then we have $ \tilde{h}(x; \emptyset) = \mathbb{E}[Y|X_1=x]=h^{\dagger}(x)$ for almost all $x$ in the support of training distribution except over a set of measure zero.  
\end{proof}

\subsection{Proof of Theorem \ref{thm: zoomin}}
\zoomin*

\begin{proof}
In this proof, we assume that all the concerned random variables $X,Y,C_t$, where $X$ is the current query and $Y$ is its label and $C_t$ is the context preceeding it, and $b(X,C_t)$ are discrete-valued for ease of exposition. Subsequently, we provide a proof for more general settings.

Since each $(X_j,Y_j)$ is sampled independently given a training environment $E$, we can conclude $I(Y; C_t | X,E)=0$. Therefore, 
$$I(Y; C_t | X,E)=0 \implies H(Y|X,E) = H(Y|X,E,C_t)$$

Observe that for all $t \in \mathbb{Z}_{+}$
\begin{equation}
   H(Y|X,E) =  H(Y|X,E,C_t) \leq H(Y|X,C_t) \leq H(Y|X, b(X, C_t)) 
    \label{eqn: proof_zoom_in_ineq1}
\end{equation} 
where $\mathbb{Z}_{+}$ is the set of all positive integers.  
The first inequality in the above follows from the fact that conditioning reduces entropy. For the second inequality, we use the following property. Consider $U,V$ as two random variables and define $W = a(V)$. 
Observe that $I(U; W | V) = 0 \implies H(U|V) = H(U|V,W) \leq H(U|W)$. 

Since the inequality above \eqref{eqn: proof_zoom_in_ineq1} holds for all $t$, we obtain

\begin{equation}
    H(Y|X,E) \leq \lim_{t\rightarrow \infty} H(Y|X,C_t) \leq \lim_{t\rightarrow \infty }H(Y|X, b(X, C_t))
    \label{eqn: proof_iid_zoom_in_ineq2_2}
\end{equation}

Next, we argue that $\lim_{t\rightarrow \infty }H(Y|X, b(X, C_t)) =  H(Y|X,E)$, which combined with \eqref{eqn: proof_iid_zoom_in_ineq2_2} yields what we intend to prove, i.e., $\lim_{t\rightarrow \infty} H(Y|X,C_t)= H(Y|X,E)$.

 For each $X=x$ and $E=e$ in the support, we argue that $b(X,C_t)\stackrel{a.s.}{\rightarrow} \theta_x^{e}$. Suppose this was not true. This implies that the probability that $P(\lim_{t\rightarrow \infty}b(X,C_t)\not= \theta_x^{e}|X=x, E=e)=\beta>0$. Since $X=x$, $E=e$ occurs with a finite probability (as $X$ and $E$ are discrete-valued and $x,e$ is in the support) say $\alpha$, then $\alpha\beta$ fraction of sequences of $b(X,C_t)$ do not converge to $\theta_x^{e}$, which contradicts the assumption that $b(X,C_t) \stackrel{a.s.}{\rightarrow} \theta_X^{E}$.

Consider a $(x,\theta)$ from the support of $(X, \theta_X^{E})$, where $X$ is the current query and $E$ is the environment from which $X$ and context preceeding it is sampled. Let us consider the distribution $P(Y|X, b(X, C_t))$

\begin{equation}
    \begin{split}
        P(Y=y|X=x, b(X, C_t)=\theta) = \frac{P(Y=y, X=x, b(X, C_t)=\theta)}{P(X=x, b(X, C_t)=\theta)}
    \end{split}
\end{equation}

We simplify $\lim_{t\rightarrow \infty }P(Y|X, b(X, C_t))$ below.

\begin{equation}
    \begin{split}
&      \lim_{t \rightarrow \infty}  P(Y=y|X=x, b(X, C_t)=\theta) = \frac{ \lim_{t\rightarrow \infty}P(Y=y, X=x, b(X, C_t)=\theta )}{ \lim_{t\rightarrow \infty}P(X=x, b(X, C_t)=\theta)} \\  
\end{split}
\end{equation}

We simplify the numerator and the denominator of the above separately below. 
\begin{equation}
\begin{split}
& \lim_{t\rightarrow \infty} P(Y=y, X=x, b(X, C_t)=\theta ) =  \lim_{t\rightarrow \infty} \sum_{e} P(Y=y, X=x, E=e, b(X, C_t)=\theta )  \\ 
&  \sum_{e} P(Y=y | X=x, E=e) \lim_{t\rightarrow \infty} P(X=x, E=e, b(X, C_t)=\theta )  \\
& =  \sum_{e} P(Y=y | X=x, E=e)  P(X=x, E=e) \lim_{t\rightarrow \infty}P(b(X, C_t)=\theta |X=x, E=e)  \\ 
    \end{split}
\end{equation} 

In the simplification above, we use the fact $Y \perp C_t | X, E$.  Since $b(X, C_t)$ converges to $\theta_x^{e}$ almost surely, the distribution  $\lim_{t\rightarrow \infty} P(b(X, C_t)=\theta |X=x, E=e) $ takes a value one if $\theta=\theta_x^e$ and zero otherwise. As a result, the above expression becomes.

\begin{equation}
    \lim_{t \rightarrow \infty}  P(Y=y, X=x, b(X, C_t)=\theta) = \sum_{e \in \mathcal{E}_{x,\theta}} P(Y=y | X=x, E=e)  P(X=x, E=e)
\end{equation}
where $\mathcal{E}_{x,\theta}$ is the set of all the environments observed conditional on $X=x$ with  $\theta_x^e =\theta$. Observe that all the environments in $\mathcal{E}_{x,\theta}$ have the same $P(Y=1|X=x,E=e)$ given by $h^{\star}(x,\theta)$.  We can write

  \begin{equation}\lim_{t \rightarrow \infty}  P(Y=1, X=x, b(X, C_t)=\theta)= h^{\star}(x, \theta) \sum_{e \in \mathcal{E}_{x,\theta}}P(X=x, E=e) 
  \label{eqn: proof_iid_zoomin_num}
  \end{equation}

We simplify $\lim_{t \rightarrow \infty} P (X=x, b(X, C_t)=\theta) $ in a similar manner to obtain
\begin{equation}
     \lim_{t \rightarrow \infty} P (X=x, b(X, C_t)=\theta) = \sum_{e \in \mathcal{E}_{x,\theta}}  P(X=x, E=e)
      \label{eqn: proof_iid_zoomin_den}
\end{equation}

We use \eqref{eqn: proof_iid_zoomin_num} and \eqref{eqn: proof_iid_zoomin_den} to obtain 

\begin{equation}
\begin{split}
     \lim_{t \rightarrow \infty}  P(Y=1|X=x, b(X, C_t)=\theta) &= \frac{\lim_{t \rightarrow \infty}  P(Y=1, X=x, b(X, C_t)=\theta)}{\lim_{t \rightarrow \infty} P(X=x, b(X, C_t)=\theta)} \\ 
     & = \frac{h^{\star}(x, \theta) \sum_{e \in \mathcal{E}_{x,\theta}}P(X=x, E=e)}{\sum_{e \in \mathcal{E}_{x,\theta}}  P(X=x, E=e)} = h^{\star}(x,\theta) 
\end{split}
\end{equation}

Therefore,  
\begin{equation} \lim_{t \rightarrow \infty}  P(Y=1|X=x, b(X, C_t)=\theta) = P(Y=1|X=x,E=e)
\label{eqn: proof_prob_eq_conv1}
\end{equation}

where $e$ is any environment in $\mathcal{E}_{x,\theta}$, i.e., it is in the support of data sampled with $X=x$ and that also satisfies $\theta_{x}^{e}=\theta$.

\begin{equation}
    \begin{split}
       & \lim_{t\rightarrow \infty} H(Y | X, b(X, C_t)) =  \sum_{x,\theta} \lim_{t\rightarrow \infty} P(X=x, b(X,C_t) = \theta) \lim_{t\rightarrow \infty} H(Y|X=x, b(X,C_t)=\theta) \\ 
       &  \sum_{x,\theta} \Big(\sum_{\tilde{e}\in \mathcal{E}_{x,\theta}} P(X=x, E=\tilde{e})\Big) \lim_{t\rightarrow \infty} H(Y|X=x, b(X,C_t)=\theta)
    \end{split}
\end{equation}

From  \eqref{eqn: proof_prob_eq_conv1}, it follows that $\lim_{t \rightarrow \infty} H(Y|X=x, b(X,C_t)=\theta) = H(Y|X=x,E=e) $, where $e$ is any environment in $\mathcal{E}_{x,\theta}$. We use this in the above to get
\begin{equation}
    \begin{split}
       & \lim_{t\rightarrow \infty} H(Y | X, b(X, C_t)) =  \sum_{x,\theta} \Big(\sum_{\tilde{e}\in \mathcal{E}_{x,\theta}} P(X=x, E=\tilde{e})\Big) H(Y|X=x, E=e) \\ 
       & =  \sum_{x,\theta} \Big(\sum_{\tilde{e}\in \mathcal{E}_{x,\theta}} P(X=x, E=\tilde{e})\Big) H(Y|X=x, E=\tilde{e})  \\
       & = \sum_{x,\tilde{e}}  P(X=x, E=\tilde{e}) H(Y|X=x, E=\tilde{e}) = H(Y|X,E)
       \end{split}
\end{equation}

\end{proof}

\begin{proof}
We now extend the previous result to setting beyond discrete random variables. In particular, we consider settings where $X, E, b(X,C_t)$ can be either discrete or continuous random variables.  In the notation to follow, we use $dP$ to denote the Radon-Nikodym derivatives. For discrete random variable, the Radon-Nikodym derivatives correspond to the standard probability mass function and for continuous random variables it would correspond to standard probability density functions. While much of the proof that follows is same as the previous proof, we repeat the arguments for completeness.

Since each $(X_j,Y_j)$ is sampled independently given a training environment $E$, we can conclude $I(Y; C_t | X,E)=0$. Therefore, 
$$I(Y; C_t | X,E)=0 \implies H(Y|X,E) = H(Y|X,E,C_t)$$

Observe that for all $t \in \mathbb{Z}_{+}$
\begin{equation}
   H(Y|X,E) =  H(Y|X,E,C_t) \leq H(Y|X,C_t) \leq H(Y|X, b(X, C_t)) 
    \label{eqn: proof_zoom_in_ineq1_2}
\end{equation} 
where $\mathbb{Z}_{+}$ is the set of all positive integers.  
The first inequality in the above follows from the fact that conditioning reduces entropy. For the second inequality, we use the following property. Consider $U,V$ as two random variables and define $W = a(V)$. 
Observe that $I(U; W | V) = 0 \implies H(U|V) = H(U|V,W) \leq H(U|W)$. 

Since the inequality above \eqref{eqn: proof_zoom_in_ineq1_2} holds for all $t$, we obtain

\begin{equation}
    H(Y|X,E) \leq \lim_{t\rightarrow \infty} H(Y|X,C_t) \leq \lim_{t\rightarrow \infty }H(Y|X, b(X, C_t))
    \label{eqn: proof_iid_zoom_in_ineq2}
\end{equation}

Next, we argue that $\lim_{t\rightarrow \infty }H(Y|X, b(X, C_t)) =  H(Y|X,E)$, which combined with \eqref{eqn: proof_iid_zoom_in_ineq2} yields what we intend to prove, i.e., $\lim_{t\rightarrow \infty} H(Y|X,C_t)= H(Y|X,E)$.

 For each $X=x$ and $E=e$ in the support except over a set of probability measure zero, we argue that $b(X,C_t)\stackrel{a.s.}{\rightarrow} \theta_x^{e}$. Suppose this was not true. Define $\Gamma$ to be the set of values of $x,e$ for which $b(X,C_t)\stackrel{a.s.}{\not\rightarrow} \theta_x^{e}$. Let $P((X,E)\in \Gamma)>0$ and the probability that $P(\lim_{t\rightarrow \infty}b(X,C_t)\not= \theta_X^{E}|(X,E) \in \Gamma)>0$. If this is true then $P(\lim_{t\rightarrow \infty}b(X,C_t)\not= \theta_x^{e})>0$ contradicts the fact that  $b(X,C_t) \stackrel{a.s.}{\rightarrow} \theta_X^{E}$. Therefore, $P((X,E)\in \Gamma)=0$.

Consider a $(x,\theta)$ from the support of $(X, \theta_X^{E})$ except from $\Gamma$, where $X$ is the current query and $E$ is the environment from which $X$ and context preceeding it is sampled. 
Let us consider the distribution $dP(Y|X, b(X, C_t))$

\begin{equation}
    \begin{split}
        dP(Y=y|X=x, b(X, C_t)=\theta) = \frac{dP(Y=y, X=x, b(X, C_t)=\theta)}{dP(X=x, b(X, C_t)=\theta)}
    \end{split}
\end{equation}

We simplify $\lim_{t \rightarrow \infty}  dP(Y=y|X=x, b(X, C_t)=\theta)$ below.

\begin{equation}
     \lim_{t \rightarrow \infty}  dP(Y=y|X=x, b(X, C_t)=\theta) = \frac{\lim_{t\rightarrow \infty} dP(Y=y, X=x, b(X, C_t)=\theta )}{\lim_{t\rightarrow \infty} dP(X=x, b(X, C_t)=\theta)}
\end{equation}

We simplify the numerator and the denominator of the above separately. 

\begin{equation}
    \begin{split}
& \lim_{t\rightarrow \infty} dP(Y=y, X=x, b(X, C_t)=\theta ) =  \lim_{t\rightarrow \infty} \int_{e} dP(Y=y, X=x, E=e, b(X, C_t)=\theta )   \\ 
&  \int_{e} dP(Y=y | X=x, E=e)  \lim_{t\rightarrow \infty} dP(X=x, E=e, b(X, C_t)=\theta ) \\
& =  \int_{e} dP(Y=y | X=x, E=e)  dP(X=x, E=e) \lim_{t\rightarrow \infty}dP(b(X, C_t)=\theta |X=x, E=e)   \\ 
    \end{split}
\end{equation} 
In the above, we use Monotone convergence theorem to swap limit and the integrals. Since $b(X, C_t)$ converges to $\theta_x^{e}$ almost surely, the distribution  $\lim_{t\rightarrow \infty} dP(b(X, C_t)=\theta |X=x, E=e) $ evaluates to probability one when $\theta= \theta_{x}^{e}$ and is zero otherwise. As a result, the above expressions become

\begin{equation}
    \lim_{t \rightarrow \infty} dP(Y=y, X=x, b(X, C_t)=\theta) = \int_{e\in \mathcal{E}_{x,\theta}} dP(Y=y | X=x, E=e)  dP(X=x, E=e) 
\end{equation}

where $\mathcal{E}_{x,\theta}$ is the set of all the environments observed conditional on $X=x$ with  $\theta_x^e =\theta$. Observe that all the environments in $\mathcal{E}_{x,\theta}$ have the same $dP(Y=1|X=x,E=e)$ given by $h^{\star}(x,\theta)$. Similarly, 

\begin{equation}
     \lim_{t \rightarrow \infty} dP (X=x, b(X, C_t)=\theta) = \int_{e \in \mathcal{E}_{x,\theta}}  dP(X=x, E=e)
\end{equation}

 As a result, we can write 

$$  \lim_{t \rightarrow \infty}  dP(Y=1, X=x, b(X, C_t)=\theta)= h^{\star}(x, \theta) \int_{e \in \mathcal{E}_{x,\theta}}dP(X=x, E=e)$$

We use this to obtain 

\begin{equation}
\begin{split}
     \lim_{t \rightarrow \infty}  dP(Y=1|X=x, b(X, C_t)=\theta) &= \frac{\lim_{t \rightarrow \infty}  dP(Y=1, X=x, b(X, C_t)=\theta)}{\lim_{t \rightarrow \infty} dP(X=x, b(X, C_t)=\theta)} \\ 
     & = \frac{h^{\star}(x, \theta) \int_{e \in \mathcal{E}_{x,\theta}}dP(X=x, E=e)}{\int_{e \in \mathcal{E}_{x,\theta}}  dP(X=x, E=e) } = h^{\star}(x,\theta) 
\end{split}
\end{equation}

Therefore,  
\begin{equation} \lim_{t \rightarrow \infty}  dP(Y=y|X=x, b(X, C_t)=\theta) = dP(Y=y|X=x,E=e)
\label{eqn: proof_prob_eq_conv}
\end{equation}

where $e$ is any environment that is in the support of data sampled with $X=x$ and that also satisfies $\theta_{x}^{e}=\theta$.

\begin{equation}
    \begin{split}
       & \lim_{t\rightarrow \infty} H(Y | X, b(X, C_t)) =  \int_{x,\theta} \lim_{t\rightarrow \infty} dP(X=x, b(X,C_t) = \theta) \lim_{t\rightarrow \infty} H(Y|X=x, b(X,C_t)=\theta) \\ 
       &  \int_{x,\theta} \Big(\int_{\tilde{e}\in \mathcal{E}_{x,\theta}} dP(X=x, E=\tilde{e})\Big) \lim_{t\rightarrow \infty} H(Y|X=x, b(X,C_t)=\theta)
    \end{split}
\end{equation}

From  \eqref{eqn: proof_prob_eq_conv}, it follows that $\lim_{t \rightarrow \infty} H(Y|X=x, b(X,C_t)=\theta) = H(Y|X=x,E=e) $, where $e$ is any environment in $\mathcal{E}_{x,\theta}$. We use this in the above to get
\begin{equation}
    \begin{split}
       & \lim_{t\rightarrow \infty} H(Y | X, b(X, C_t)) =  \int_{x,\theta} \Big(\int_{\tilde{e}\in \mathcal{E}_{x,\theta}} dP(X=x, E=\tilde{e})\Big) H(Y|X=x, E=e) \\ 
       & =  \int_{x,\theta} \Big(\int_{\tilde{e}\in \mathcal{E}_{x,\theta}} dP(X=x, E=\tilde{e})\Big) H(Y|X=x, E=\tilde{e})  \\
       & = \int_{x,\tilde{e}}  dP(X=x, E=\tilde{e}) H(Y|X=x, E=\tilde{e}) = H(Y|X,E)
       \end{split}
\end{equation}

\end{proof}

\subsection{Proof of Theorem \ref{thm: partialzoomin}}
\partialzoomin*

\begin{proof}

Let us consider the setting where the context is of length one. We denote the current query as $X$ with corresponding label $Y$ and environment $E$. The example in the context is $\tilde{X}$ which has corresponding label $\tilde{Y}$ and it shares the same environment $E$. Recall that as part of the context, the learner only sees $\tilde{X}$ and not $\tilde{Y}$. Both $Y$ and $E$ are real-valued scalars and $X$ is a $d$ dimensional vector.  

Following the assumption in the theorem, the distribution of $(\tilde{X},\tilde{Y},X,Y,E)$ is Markov with respect to a Bayesian network.  We first establish that $E$ cannot be a child of any variable in the directed acyclic graph (DAG). The assumption $(X,Y) \perp (\tilde{X}, \tilde{Y}) | E $  implies
$ X \perp \tilde{X} | E$ and $Y\perp \tilde{Y}|E$. Suppose $E$ is a child variable of $Y$. Due to the symmetry, $(X,Y,E)$ and $(\tilde{X},\tilde{Y}, E)$ follow the same distribution. As a result, $E$ is also a child variable of $\tilde{Y}$, which implies $Y\not\perp \tilde{Y}|E$ (since $E$ is a collider on the path from $Y$ to $\tilde{Y}$). This contradicts $Y\perp \tilde{Y}|E$. Suppose $E$ is a child variable of some component of $X$ say $X^i$. Due to symmetry, $E$ is also a child variable of $\tilde{X}^{i}$, which implies $X^{i} \not\perp \tilde{X}^{i} |E$. This contradicts $X \perp \tilde{X} | E$. Therefore, $E$ cannot be a child of any of the variables in the DAG.

Since both $X$ and $E$ form the Markov blanket of $Y$, there are two possible cases. Either $E$ is directly connected to $Y$ or $E$ is connected to $Y$ through some element of $X$. 

In the first case, $E$ can only have an arrow into $Y$ and not the other way around as $E$ is not a child of any other node. Let us consider the setting when $E$ is one of the parents of $Y$ and denote it as $E \rightarrow Y$.  Since $X$ ($\tilde{X}$) is on the  Markov Blanket  of $Y$ ($\tilde{Y}$), we claim that each component of $X$ is either a parent of $Y$ or a child of $Y$.  Suppose this was not the case. This implies that there exists a component of $X$ say $X^i$, which is on the Markov Blanket as a parent of $E$. But that would make $E$ a child of $Y$. However, $E$ cannot be a child variable as shown above. As a result, each component of $X$ is either a parent or a child of $Y$. We now consider two subcases.

Let us consider the setting when there exists a child $X^{i}$  of $Y$. Observe that $\tilde{X}^{i}$ is a child of $\tilde{Y}$ and it has a path to $E$ and as a result it has a path to $Y$. This path from elements of $\tilde{X}^{i}$ to $\tilde{Y}$ passes through $E$. This path has no colliders and does not contain any element of $X$ on it (We show this case in~\Cref{figure:proof_partial_a}). As a result, $Y \not\perp \tilde{X}^{i} | X$. Thus $I(Y;\tilde{X}|X)>0$ (use chain rule of mutual information). 

    Let us consider the other setting when each $X^i$ is a parent of $Y$ (shown in~\Cref{figure:proof_partial_b}). In this case, $E$ has to have a path to some element of $X$, say $X^j$ as otherwise $E\perp X$, which contradicts the assumption that $E\not\perp X$. Consider the path $\tilde{X}^{j}$ to $E$ to $Y$. Observe that this path is not blocked. As a result,  $I(Y;\tilde{X}|X)>0$. 
    
    Let us consider the other possibility when $Y$ is connected to $E$ through $X$. Here the only way this is possible is if some element  of $X$ say $X^{i}$ is a child of $Y$ and $E$ is a parent of that element (as shown in~\Cref{figure:proof_partial_c}). Therefore, we know that $\tilde{X}^{i}$ is connected to $Y$ through $E$ and $X^{i}$. 
    
    Observe that this path from $\tilde{X}^{i}$ to $Y$ is not blocked as $X^{i}$ is a collider. Therefore, $I(Y; \tilde{X}|X)>0$. 
    We showed the result so far assuming that the context length was one. Suppose that the context has $k-1$ examples denoted as $C_k=[X_{1}\cdots, X_{k-1}]$. The chain rule of mutual information tells us $I(Y; C_k | X) = I(Y; X_{k-1} | X) + I(Y; C_{k-1}|X, X_{k-1})$. The proof above already demonstrates that the first term $I(Y; X_{k-1} | X)$ is strictly positive. Since mutual information is non-negative, we can conclude that $I(Y; C_k | X)>0$. 

    Next, we want to argue that entropy strictly reduces as context length increases. In other words, $$H(Y|X,C_k) < H(Y|X,C_{k-1}) \iff I(Y; X_k | X, C_{k-1})>0$$

  We want to show $Y \not\perp X_k | (X, C_{k-1})$.  In the proof above, we had three cases shown in~\Cref{figure:proof_partial}. In each of these cases, we argued that the path from $X_k$ to $Y$ is not blocked. Even if we condition on contexts $C_{k-1}$ this continues to be the case. In the first two cases, the path from $X_k$ to $Y$ is direct and does not contain any element from the conditioning set. In the third case, the direct path involves a collider $X$ from the conditioning set and thus is also not blocked. As a result, $Y \not\perp X_k | (X, C_{k-1})$. This completes the proof.

    \begin{figure}
  \centering
    \begin{subfigure}[b]{0.3\textwidth}
    \centering
    \resizebox{\textwidth}{!}{%
    \begin{tikzpicture}[->,>=stealth',node distance=1.5cm,
                    thick,main node/.style={circle,draw,font=\footnotesize, minimum size=0.8cm}]
  \node[main node] (E) {$E$};
  \node[main node] (Y) [below left of=E] {$Y$};
   \node[main node] (Yp) [below right of=E] {$\tilde{Y}$};
  \node[main node] (X) [below right of=Y] {\(X^i\)};
   \node[main node] (Xp) [below right of=Yp] {\(\tilde{X}^{i}\)};
  \path[every node/.style={font=\footnotesize}]
    (E) edge node {} (Y)
    (E) edge node {} (Yp)
    (Y) edge node {} (X) 
    (Yp) edge node {} (Xp);
\end{tikzpicture}}
    \caption{Case 1.}  \label{figure:proof_partial_a}
  \end{subfigure}
  \hfill
  \begin{subfigure}[b]{0.3\textwidth}
    \centering
    \resizebox{\textwidth}{!}{%
    \begin{tikzpicture}[->,>=stealth',node distance=1.5cm,
                    thick,main node/.style={circle,draw,font=\footnotesize, minimum size=0.8cm}]
 \node[main node] (E) {$E$};
  \node[main node] (X) [left of=E] {$X$};
   \node[main node] (Xp) [right of=E] {$\tilde{X}$};
  \node[main node] (Y) [below left of=E] {\(Y\)};
   \node[main node] (Yp) [below right of=E] {\(\tilde{Y}\)};
  \path[every node/.style={font=\footnotesize}]
    (E) edge node {} (Y)
    (E) edge node {} (Yp)
    (E) edge node {} (X) 
    (E) edge node {} (Xp) 
    (X) edge node {} (Y) 
    (Xp) edge node {} (Yp);
\end{tikzpicture}}
    \caption{Case 2.}
   \label{figure:proof_partial_b}
  \end{subfigure}
  \hfill
  \begin{subfigure}[b]{0.3\textwidth}
    \centering
    \resizebox{\textwidth}{!}{%
        \begin{tikzpicture}[->,>=stealth',node distance=1.5cm,
                    thick,main node/.style={circle,draw,font=\footnotesize, minimum size=0.8cm}]
 \node[main node] (Y) {$Y$};
  \node[main node] (E) [right of=Y] {$E$};
   \node[main node] (X) [below right of=Y] {$X^{i}$};
  \node[main node] (Yp) [right of=E] {$\tilde{Y}$};
   \node[main node] (Xp) [below right of=E] {\(\tilde{X}^{i}\)};
  \path[every node/.style={font=\footnotesize}]
    (E) edge node {} (X)
    (Y) edge node {} (X) 
    (E) edge node {} (Xp) 
    (Yp) edge node {} (Xp);
\end{tikzpicture}
}
    \caption{Case 3.}
    \label{figure:proof_partial_c}
  \end{subfigure}
  \caption{Illustrating the different key cases for~\Cref{thm: partialzoomin}.
   }
  \label{figure:proof_partial}
\end{figure}
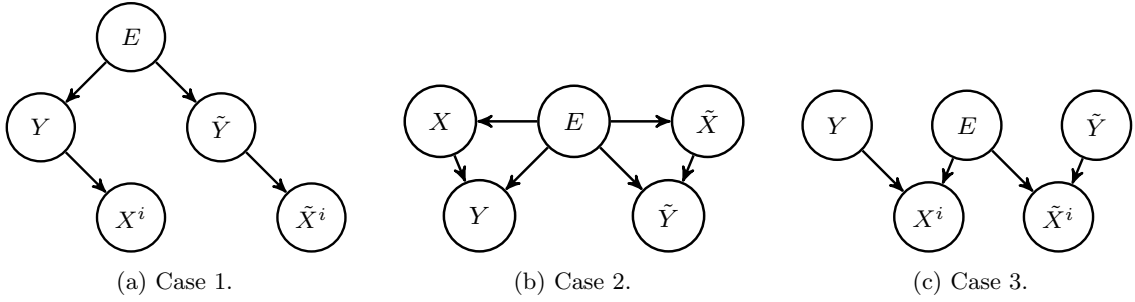
\end{proof}

\paragraph{Remark on the~\Cref{thm: partialzoomin}} It is possible to extend~\Cref{thm: partialzoomin} to the case when only a subset of $X$ and $E$ form the Markov blanket. Observe that the analysis of Case a) and Case c) in~\Cref{figure:proof_partial_a}, ~\Cref{figure:proof_partial_c} does not change. The analysis of Case b) is more nuanced now. In Case b), we used the fact that $E$ is connected to $X$ that is on the Markov blanket. This need not be the case if only a subset of $X$ is on the Markov blanket. Suppose $X_{\mathsf{MB}}$ denote the set of $X$ that are on the Markov Blanket. If $E$ is connected to any member of $X_{\mathsf{MB}}$, the same analysis as Case b) continues to hold. Consider the case when $E$ is connected to some other member of $X$ that is not in $X_{\mathsf{MB}}$. Denote this member as $X^{i}$. Observe that the same element $\tilde{X}^{i}$ from $\tilde{X}$ will have a direct path into $Y$ through $E$ that is not blocked. As a result, even in this case conditioning on $\tilde{X}$ helps.

\subsection{Proof of Theorem \ref{thm: oodzoomin}}

\oodzoomin*

\begin{proof} The learning algorithm works as follows.  
For each $e,y$ pair in the training data, define the set of $x's$ as $\mathcal{D}_{x}^{e,y}$. Maximize the likelihood of  $\mathcal{D}_{x}^{e,y}$ assuming that the underlying distribution is Gaussian.  This can be stated as 
    $$\hat{\mu}_e^y, \hat{\Sigma}_e^y = \argmin_{\mu_e^y, \Sigma_y^e} \Big( \sum_{x \in \mathcal{D}_{x}^{e,y}}\Big[\|x-\mu_{e}^{y}\|_{(\Sigma_{e}^{y})^{-1}}^2 \Big]-\log(\mathsf{det}(\Sigma_{y}^e))\Big)$$

The solution to the above are standard sample mean based estimators of means and covariance.  Also, use a sample mean based estimator to estimate the probability of each class in environment $e$ and denote it as $\hat{p}_e^y$.  Define $\hat{\gamma}_e = [(\hat{p}_e^{0}, \hat{\mu}_e^{0}, \hat{\Sigma}_e^{0}), (\hat{p}_e^{1}, \hat{\mu}_e^{1}, \hat{\Sigma}_e^{1})]$. The model at test time works as follows. 

\begin{itemize}
    \item We are given samples $\mathcal{D}_x^{e'}$ at test time from some environment $e' \in \mathcal{E}_{te}$. Estimate the parameters of Gaussian mixture model with two mixture components to maximize the likelihood of observing $\mathcal{D}_x^{e'}$. We denote the estimated parameters as  $\theta_{e'}=[p_{e'}, \mu_{e'}, \Sigma_{e'}, \tilde{\mu}_{e'}, \tilde{\Sigma}_{e'}]$. Define a permutation of $\theta_{e}^{'}$ as $\beta_{e'} = [\tilde{p}_{e'}, \tilde{\mu}_{e'}, \tilde{\Sigma}_{e'}, p_{e'}, \mu_{e'}, \Sigma_{e'}]$. 
    \item Find the closest environment to the estimated parameters in the training set. 
    \begin{equation}
     \min_{e \in \mathcal{E}_{tr}} \bigg(\min\{\|\theta_{e'} - \hat{\gamma}_e \|,\|\beta_{e'} -\hat{\gamma}_e\| \}\bigg)
     \label{eqn: algo_ood}
    \end{equation}
    Suppose $\tilde{e}$ is the closest training environment that solves the above. If $\theta_{e'}$ is closer to $\hat{\gamma}_{\tilde{e}}$ than $\beta_{e'}$, then $p_{e'},\mu_{e'}, \Sigma_{e'}$ correspond to the label $0$ and $\tilde{p}_{e'}, \tilde{\mu}_{e'}, \tilde{\Sigma}_{e'}$ correspond to the label $1$.  For the query $x$, the probability assigned to label $0$ is 

    $$c(x) = \frac{p_{e'}e^{-\|x-\mu_{e'}\|_{(\Sigma_{e'})^{-1}}^2}}{p_{e'} e^{-\|x-\mu_{e'}\|_{(\Sigma_{e'})^{-1}}^2} + \tilde{p}_{e'}e^{-\|x-\tilde{\mu}_{e'}\|_{(\tilde{\Sigma}_{e'})^{-1}}^2}}$$

    If $\beta_{e'}$ is closest to this environment, then $p_{e'}, \mu_{e'}, \Sigma_{e'}$ correspond to the label $1$ and $\tilde{p}_{e'}, \tilde{\mu}_{e'}, \tilde{\Sigma}_{e'}$ is the label $0$. For the query $x$, the probability assigned to label $0$ is $1-c(x)$.

\end{itemize}

For the training environments, in the limit of infinitely long contexts the estimated parameters take exact values, i.e., $\hat{\gamma}_e = \gamma_e$, for all $e\in \mathcal{E}_{tr}$.

For the test environment, the true set of parameters that generate the data are $\gamma_{e'}$, where $\gamma_{e'}=\big[(p_{e'}^{0}, \mu_{e'}^{0}, \Sigma_{e'}^{0}), (p_{e'}^{1}, \mu_{e'}^{1}, \Sigma_{e'}^{1})\big]$.  Define the permutation of $\gamma_{e'}$ as $\delta_{e'} =\big[ (p_{e'}^{1}, \mu_{e'}^{1}, \Sigma_{e'}^{1}), (p_{e'}^{0}, \mu_{e'}^{0}, \Sigma_{e'}^{0}) \big]$. 

There can be two types of test environments. One in which the mean and covariance for both classes are identical. The method above assigns a probability of $\frac{1}{2}$ to both the classes, which is the Bayes optimal prediction. Let us consider the latter environments, where the class conditional parameters for $x$ are not the same. In the limit of infinitely long contexts at test time, there are two possible values $\theta_{e'}$ can take, either $\theta_{e'}=\gamma_{e'}$ or $\theta_{e'}=\delta_{e'}$. This follows from identifiability of Gaussian mixtures, \cite{yakowitz1968identifiability}. 

Consider the first case, $\theta_{e'} = \gamma_{e'}$. In this case, the \eqref{eqn: algo_ood} becomes 
$$ \min_{e \in \mathcal{E}_{tr}} \bigg(\min\{\|\gamma_{e'} - \gamma_e \|,\|\delta_{e'} -\gamma_e\| \}\bigg).$$

Suppose some environment $\tilde{e}$ solves the above optimization.  Following the assumption in we know that  $\gamma_{e'}$ falls in the Voronoi region of some $\gamma_{\tilde{e}}$ and thus $\gamma_{e^{'}}$ is closer to $\gamma_e$ than $\delta_e$. As a result, $p_{e'}^{0}, \mu_{e'}^{0}, \Sigma_{e'}^{0}$ is associated with class $0$, which is actually correct and thus the final predictor would match the Bayes optimal predictor for the test environment. In the second case, $\theta_{e'}= \delta_{e'}$. Therefore, $\beta_{e'}=\gamma_{e'}$ and $p_{e'}^{1}, \mu_{e'}^{1}, \Sigma_{e'}^{1}$ would be correctly associated with class one thus leading to Bayes optimal predictions. This completes the argument we set out to prove. 

We now briefly explain how the method fails if test parameter is outside the Voronoi cell of training parameters. Suppose $\theta_{e'} = \gamma_{e'}$ but $\gamma_{e'}$ is in Voronoi region of some $\delta_{e}$. In this case, $\beta_{e'}$ would be closest to $\gamma_{e}$ and $p_{e'}^{0}, \mu_{e'}^{0}, \Sigma_{e'}^{0}$ would be incorrectly associated with class $1$. This  shows that beyond the Voronoi region the proposed algorithm fails.

\end{proof}

\subsection{Extension of Theorem \ref{thm: oodzoomin}}

In the previous theorem, we assumed that $g$ is identity. We now describe how the result can be extended to general non-linear mixing maps $g$. For this result, we leverage the theoretical results from identifiable variational autoencoders (i-VAE) \citep{khemakhem2020variational}.  

\paragraph{A short review of identifiable variational autoencoders}
We are provided with observations $x$'s that are generated from a latent variable $z$ using an injective map $g$, where $x\leftarrow g(z)$. The theory of i-VAE provides with a method and the conditions under which the underlying true latent variables $z$ can be identified up to permutation and scaling.  In i-VAEs, it is assumed that along with each sample $x$, we are provided with auxiliary information, which they term as $u$. For our results, auxiliary information is available to us in the form of the environment index and the label of the data point.   In the theory of i-VAE,  the distribution of the latent variables are assumed to follow a conditionally factorial exponential distribution stated as follows.

\begin{equation}
    p_{T, \lambda}(z | u) =   \prod_{i} \frac{Q_i(z_i)}{M_i(u)} \exp\bigg[\sum_{j=1}^{k}T_{i,j}(z_i)\lambda_{i,j}(u) \bigg]
\end{equation}

where $T_i = (T_{i,1}, \cdots, T_{i,k})$ are the sufficient statistics, $\lambda_i(u) = (\lambda_{i,1}(u), \cdots, \lambda_{i,k}(u))$ are the parameters of the distribution that vary with $u$, $Q_i$ is a base measure and $M_i$ is a normalizing constant. We concatenate $T_i$'s and $\lambda_{i}'s$ across $d$ latent dimensions to make construct $dk$ dimensional vectors denoted as $\lambda(u)$ and $T(z)$.  Thus the data generation process is summarized as 

\begin{equation}
\begin{split}
    & z\sim p_{T, \lambda}(\cdot |u) \\
    & x \leftarrow g(z)
\end{split}
    \label{eqn: dgp_ivae}
\end{equation}
where $g,T, \lambda$ are the parameters.  We now revisit the data generation process that we consider and explain how it falls under the umbrella of the data generation processes considered in i-VAE. For all $e \in \mathcal{E}$, 
\begin{equation}
\begin{split} 
  &  z|y,e \sim \mathcal{N}(\mu_e^y, \Sigma_e^y) \\ 
  &  x \leftarrow g(z)
\end{split}
\label{eqn_app: dgp_ood}
\end{equation}
where the latent variables $z$ are sampled conditional on the label $y$ and environment $e$ from a Normal distribution whose mean and covariance depend on both $y,e$. We further assume that the covariance matrix has a diagonal structure as stated below.

\begin{assumption}
    Each $\Sigma_{e}^y$ is a diagonal matrix. 
\end{assumption}

Since $\Sigma_{e}^y$ is a diagonal matrix, we denote the $i^{th}$ diagonal element as $(\sigma_e^y(i))^2$. Similarly, the $i^{th}$ component of $\mu_e^y$ is denoted as $\mu_e^y(i)$. Observe that the distribution of $z$ conditional on $y,e$ belongs to the family conditionally factorial exponential distributions studied in i-VAE. If we substitute $Q_{i}(z_i) = \frac{1}{\sqrt{2\pi}}$, $M_i(y,e) = e^{\big((\mu_e^{y}(i))^2/(\sigma_{e}^y(i))^2\big)}$, $\lambda_{i,1}(y,e) = \frac{2\mu_{e}^{y}(i)}{(\sigma_{e}^{y}(i))^2}$, $\lambda_{i,2}(y,e) = -\frac{1}{(\sigma_{e}^{y}(i))^2}$, $T_{i,1}(z)=z$ and $T_{i,2}(z) = z^2$, then we obtain the distribution of $z$ described by \eqref{eqn_app: dgp_ood}.

\begin{definition}
We define an equivalence relation between sets of parameters of the model as follows. 
\begin{equation}
    (g,T,\lambda) \sim (\tilde{g}, \tilde{T}, \tilde{\lambda}) \iff 
    \exists A,c \; |\; T(g^{-1}(x)) = A \tilde{T}(\tilde{g}^{-1}(x)) + c, \forall x \in \mathcal{X}
\end{equation}
If $A$ is invertible, then we denote the relation by $\sim_A$. If $A$ is a block permutation matrix, then we denote it by $\sim_P$. 
\end{definition}

\begin{theorem}
\label{thm: lin_id_ivae}
    Assume that the data is sampled from the data generation in equation \eqref{eqn: dgp_ivae} according to with parameters $(g,T,\lambda)$. Assume the following holds
    \begin{itemize}
        \item The mixing function $g$ is injective
        \item The sufficient statistics $T_{i,j}$ are differentiable almost everywhere, and $(T_{i,j})_{1\leq j \leq k}$ are linearly independent on any subset of $\mathcal{X}$ of measure greater than zero. 
        \item There exists $dk+1$ distinct points $u^{0}, \cdots, u^{dk}$ such that the matrix 
        $$L = (\lambda(u_1)-\lambda(u_0), \cdots, \lambda(u_{dk}) -\lambda(u_0))$$
        of size $dk\times dk$ is invertible. 
    \end{itemize}
    then the parameters $(g, T,\lambda)$ are $\sim_A$ identifiable. 
\end{theorem}

\begin{theorem}
\label{thm: perm_id_ivae}
    Assume the hypotheses of the~\Cref{thm: lin_id_ivae} holds, and $k\geq 2$. Further assume: 
    \begin{itemize}
    \item The sufficient statistics $T_{i,j}$ are twice differentiable.
    \item The mixing function $g$ has all second order cross derivatives. 
    \end{itemize}
    then the parameters $(g,T, \lambda)$ are $\sim_P$ identifiable. 
\end{theorem}

We can leverage the above two theorems (\Cref{thm: lin_id_ivae}, \Cref{thm: perm_id_ivae} and Theorem 4 from \cite{lachapelle2022disentanglement}) and arrive at the following corollary for the Gaussian data generation process from \eqref{eqn_app: dgp_ood}.

\begin{theorem}
\label{thm: perm_id_normal}
    If the data generation process follows \eqref{eqn_app: dgp_ood}, where $g$ is injective and has all second order cross derivatives. Suppose there exist $2d+1$ points $u^{0} = (y_0,e_0), \cdots, u^{2d} = (y_{2d},e_{2d})$ in the support of $(y,e)$ observed in training distribution such that 
   $$  (\lambda(u_1)-\lambda(u_0), \cdots, \lambda(u_{2d}) -\lambda(u_0))$$
   is invertible. If $p_{g,T,\lambda}(\cdot | y,e) = p_{\tilde{g},\tilde{T},\tilde{\lambda}}(\cdot | y,e)$ for all $y,e$ in the support of $(y,e)$ in the training distribution, then $\tilde{z} = \Lambda \Pi z + r$, where $\tilde{z} = \tilde{g}^{-1}(x)$ and $z = g^{-1}(x)$. 
\end{theorem}

\begin{proof}

We equate the probability of observations $x$ under two models $g, T, \lambda$ and $\tilde{g}, \tilde{T}, \tilde{\lambda}$ for each $y,e$. Consider a $z\sim p_{T,\lambda}(\cdot |y,e)$ and the corresponding $x = g(z)$.  These $x$'s follow $p_{\tilde{g},\tilde{T},\tilde{\lambda}}(\cdot |y,e)$ since $p_{g,T,\lambda}(\cdot |y,e) = p_{\tilde{g},\tilde{T},\tilde{\lambda}}(\cdot|y,e)$. 
Define $\tilde{z} = \tilde{g}^{-1}(x)$ and these $\tilde{z}$ follow $p_{\tilde{T}, \tilde{\lambda}}(\cdot|y,e)$. We can write $\tilde{z} = a(z)$, where 
$a = \tilde{g}^{-1}\circ g$.

Observe $ p_{z}(z|y,e) = p_{\tilde{z}}(a(z)|y,e) \mathsf{det}(Da(z))$

    \begin{equation}
    \begin{split}
        \log p_z\big(z|y_k,e_k\big) = \log\big(p_{\tilde{z}}(a(z)|y_k,e_k)\big) + \log\mathsf{det}(Da(z)) \\  
         \log  p_z\big(z|y_0,e_0\big) = \log\big(p_{\tilde{z}}(a(z)|y_0,e_0)\big) + \log\mathsf{det}(Da(z)) \\ 
         \log p_z\big(z|y_k,e_k\big) -  \log \big(p_z(z|y_0,e_0)\big) =  \log\big(p_{\tilde{z}}(a(z)|y_k,e_k)\big) - \log\big(p_{\hat{z}}(a(z)|y_0,e_0)\big)
    \end{split}
    \end{equation}

    Substituting the exponential form we obtain that $$T(z)^{\top}[\lambda(y_k,e_k) -\lambda(y_0,e_0))]  = T(\tilde{z})^{\top}[\tilde{\lambda}(y_k,e_k) -\tilde{\lambda}(y_0,e_0))] $$

If we use sufficient variability conditions, we obtain  $T(z) = A T(\tilde{z}) +c$. We now use the fact that sufficient statistics $T(z) = (z,z^2)$ are minimal to conclude that
$$T(z) = A T(\tilde{z}) + c$$
where $A$ is invertible. In the above, we use the line of reasoning used in in the proof of Theorem 4 in \citep{lachapelle2022disentanglement}.

After this point, we leverage ~\Cref{thm: perm_id_ivae} to conclude that 

$$T_i(z_i) =  AT_{j}(\tilde{z}_j) + c$$

We can expand the above to write 

$$\begin{bmatrix}
    \tilde{z}_j \\
    \tilde{z}_j^2
\end{bmatrix} = D\begin{bmatrix}
    z_i \\
    z_i^2
\end{bmatrix} + e$$

Note that the above relationship holds for all $z\in \mathcal{Z}$. If $\tilde{z}_j$ depends on $z_i^2$, then $\tilde{z}_j^2$ would be a degree four polynomial in $z_i$ and it would be equated to a degree $2$ polynomial $z_i$ stated in the RHS. This cannot be true for all $z_i$ in the support. As a result, $\tilde{z}_j$ is a scalar multiple of $z_i$. Since for every $i$ there is such a $j$, it follows that $\tilde{z} = \Lambda \Pi z +r$.

\end{proof}

\begin{theorem}(\textbf{Zoom-in [ood]})
  Consider the data generation process in \eqref{eqn_app: dgp_ood}. We make a few additional assumptions on the data generation stated below.
  \begin{itemize}
      \item Each $\Sigma_e^y$ is a diagonal matrix
      \item  There exist $2d+1$ points $u^{0} = (y_0,e_0), \cdots, u^{2d} = (y_{2d},e_{2d})$ in the support of $(y,e)$ observed in training distribution such that 
   $$  (\lambda(u_1)-\lambda(u_0), \cdots, \lambda(u_{2d}) -\lambda(u_0))$$
   is invertible. 
   \item $g$ is injective and has all second order cross derivatives. 
  \end{itemize}

  Under the above assumptions, we can guarantee that there exists an in-context learning algorithm that generates Bayes optimal predictions for all the test environments that fall in Voronoi cells of training parameters weighted by a certain vector.
\end{theorem}

\begin{proof}

The training proceeds as follows.  Train an autoencoder on training data under the constraint that the output of the encoder follow a Gaussian distribution with independent components conditional on each $y,e$. This is stated as the following minimization.

 \begin{equation}
       \hat{g}, \hat{f}, \hat{\mu}_e^y, \hat{\Sigma}_e^y =  \arg\min_{\tilde{g}, \tilde{f}, \{\mu_e^{y},\Sigma_e^y\} } \mathbb{E}[\|(\tilde{g}\circ \tilde{f}(x) - x)\|^2] + \alpha\sum_{y,e} \mathsf{KL}\Big(p_{\tilde{z}}(\cdot| y,e) \; \| \; \mathcal{N}(\mu_e^{y},\Sigma_e^y)\Big)
    \end{equation}
    where $\tilde{z} = \tilde{f}(x)$, $p_{\tilde{z}}(\cdot| y,e)$ is the distribution of $\tilde{z}$. The first term is standard reconstruction loss and the second term is the KL divergence between distribution of $\tilde{z}$ and a Normal distribution with independent components.  Also, estimate the class probabilities for each environment and denote them as $\hat{p}_e^y$.     Similar to the proof of~\Cref{thm: oodzoomin} define $\hat{\gamma}_e = [(\hat{p}_e^y,\hat{\mu}_e^y, \hat{\Sigma}_e^y )_{y\in \{0,1\}} ]$

The model at test time works as follows. We first use the  trained encoder $\hat{f}$ and generate $\tilde{z}$ for test time inputs. After this the model operates in exactly the same way on $\tilde{z}'s$ as in the proof of~\Cref{thm: oodzoomin}. Basically the output of encoder takes place of raw $x$'s in the procedure described in proof of~\Cref{thm: oodzoomin}. 

The assumptions in this theorem along with following i) $\tilde{z}$ follows a Gaussian distribution with independent components, ii) $g(\tilde{z})$ follows distribution of $x$ conditional on $y,e$ for each $y,e$, implies we can use the previous result in~\Cref{thm: perm_id_normal} to conclude that $\tilde{z} = \Lambda \Pi z + r$. 
Observe that $\tilde{z}$ also follows a Gaussian distribution with independent components conditional on each $y,e$.  In the limit of infinitely long contexts, $\hat{\gamma}_e$ is equal to scaled means of original training environments and covariances also scaled componentwise according to the transform $\Lambda \Pi$.  We can now apply the previous~\Cref{thm: oodzoomin} on $\tilde{z}'s$ as follows. If the parameters of the test environment are in the Voronoi cell of the train distribution of $\tilde{z}'s$, then the procedure described above continues to generate Bayes optimal predictions in those environments.

\end{proof}

\subsection{Comparing ICRM and ERM under the lens of invariance}

The label $y$ is related to $x_1$ and mean of $x_2$ in environment $e$ as follows.  
\begin{equation}
    y \leftarrow \alpha x_1 + \beta \mu_2^e + \varepsilon
\end{equation}

ERM learns a linear model on features $x = (x_1,x_2)$. The closed form solution for linear regression is $\Lambda_{xx}^{-1}\rho_{xy}$, where $\Lambda_{xx} = \mathbb{E}[xx^{\top}]$, which is assumed to be invertible, and $\rho_{xy}= \mathbb{E}[xy]$.  The covariance matrix of $x$ is defined as $\Sigma_{xx} = \begin{bmatrix}
\sigma_1^2 \;\; \sigma_{12} \\
\sigma_{12} \;\; \sigma_2^2
\end{bmatrix}.$

\begin{proposition}
Let $\mathbb{E}[x_1|e]=0$ for all $e$. If $\Sigma_{xx}$ is invertible, $\beta\not=0$, $\sigma_{12}\not=0$, $\mu_2^e\not=0$ for some $e \in \mathcal{E}_{tr}$, then the coefficient estimated by ERM for $x_1$ is not the same as the invariant coefficient $\alpha$. 
\end{proposition}

\begin{proof}
    We compute $\rho_{xy}$ first. 
    \begin{equation}
    \begin{split}
        \rho_{xy} & = \begin{bmatrix}
            \alpha \mathbb{E}[x_1^2] + \beta \mathbb{E}[\mu_2^ex_1] \\ 
            \alpha \mathbb{E}[x_1x_2] + \beta \mathbb{e}[\mu_2^e x_2] 
        \end{bmatrix}  \\
       & =\alpha \begin{bmatrix}
             \sigma_1^2  \\ 
            \sigma_{12} + \frac{\beta}{\alpha} \delta
        \end{bmatrix},
    \end{split}
    \end{equation}
  where $\delta =  \mathbb{E}[(\mu_2^e)^2].$

  Next, we compute $\Lambda_{xx}$. 

  \begin{equation}
      \Lambda_{xx} = \begin{bmatrix}
          \sigma_{1}^2 \;\; \sigma_{12} \\ 
          \sigma_{12} \;\;  \sigma_2^2 + \delta 
      \end{bmatrix}.
  \end{equation}

The solution to ERM is 

\begin{equation}
\begin{bmatrix}
    \alpha' \\ 
    \beta'  
\end{bmatrix} = 
\frac{\alpha}{(\sigma_2^2+\delta)\sigma_1^2-\sigma_{12}^2}\begin{bmatrix}
      \sigma_2^2 + \delta      \;\; -\sigma_{12} \\ 
          -\sigma_{12} \;\;  \sigma_{1}^2
\end{bmatrix}\begin{bmatrix}
             \sigma_1^2  \\ 
            \sigma_{12} + \frac{\beta}{\alpha} \delta
        \end{bmatrix}.
\end{equation}

Simplifying the above, we obtain the coefficient for $x_1$ to be 
\begin{equation}
  \alpha' =   \alpha -\frac{\sigma_{12}\beta\mathbb{E}[(\mu_2^e)^2]}{\sigma_{1}^2\big(\sigma_2^2+\mathbb{E}[(\mu_2^e)^2]\big)-\sigma_{12}^2}. 
\end{equation}

Owing to the assumptions, $\beta\not=0, \sigma_{12}\not=0$ and $\mu_2^e$ for some $e$ we obtain that the second term in the above is not zero. As a result, the estimate computed by ERM for $\alpha$ is biased. 
\end{proof}

\begin{proposition}
Let $\mathbb{E}[x_1|e]=0$ for all $e$. If $\Sigma_{xx}$ is invertible, $\beta\not=0$, $\sigma_{12}\not=0$, $\mu_2^e\not=0$. The error of ERM in test environment increases in $\sigma_1^2$
\end{proposition}

\begin{proof}
    The error of ERM is given as 

    \begin{equation}
    \begin{split}
    &    \mathbb{E}[(\alpha x_1 + \beta \mu_{2}^{e} - \alpha'x_1 - \beta'x_2)^2] + \sigma^2_{\varepsilon} \\ 
    & = (\alpha-\alpha^{'})^2\sigma_1^2 + \beta^2 \mathbb{E}[(\mu_2^e)^2]  + (\beta')^2\mathbb{E}[x_2^2] - 2\beta\beta^{'}\mathbb{E}[(\mu_2^e)^2] -2(\alpha-\alpha')\beta\sigma_{12} + \sigma^2_{\varepsilon}, 
    \end{split}
    \end{equation}
where $\sigma_{\varepsilon}^2$ is the variance of the noise variable $\varepsilon$. If we take the derivative of the above error w.r.t $\sigma_1^2$, we obtain $(\alpha-\alpha')^2$, which is positive. This completes the proof. 
  
\end{proof}

ICRM learns a linear model on $(x_1,x_2, \mu_1^e, \mu_2^e)$. We study two settings to analyze the error of ICRM at test time. If at test time, the model has seen sufficiently long contexts, then it knows the means corresponding to $x_1$ and $x_2$ and the model achieves the test error of $\sigma_{\varepsilon}^2$. On the other hand, if the context is empty, then also note that the expected error of the model is $\beta^2 \|\mu_2^{e'}\|^2$ (assuming the model uses a default value of zero for the mean in the absence of any context), where $\mu_{2}^{e'}$ is the mean of $x_2$ in environment $e'$. Since the error of ICRM in the absence of any context is independent of variance of $x_1$, the error of ERM can be much worse than that of ICRM in this setting as well.

Moving forward let us consider a more general setting. 

\begin{equation}
\begin{split}
   & y = p(x_1,\mu_2^e) + \varepsilon,  \\ 
   & x_2= q(\mu_2^e, \vartheta),
\end{split}
\label{eqn: dgp_inv_nonlinear}
\end{equation}

where $p(\cdot)$ and $q(\cdot)$ are maps (potentially non-linear), $\varepsilon$ and $\vartheta$ are independent zero mean noise variables. Following the same line of thought as the above example. ICRM learns a non-linear model on $(x_1,x_2, \mu_1^e, \mu_2^e)$ and learns $\mathbb{E}[y|x_1,x_2,\mu_1^e, \mu_2^e]$. From \eqref{eqn: dgp_inv_nonlinear}, it follows that $$y\perp (x_2,\mu_1^e) | (x_1,\mu_2^e) \implies \mathbb{E}[y|x_1,x_2,\mu_1^e, \mu_2^e] = \mathbb{E}[y|x_1,\mu_2^e]= p(x_1,\mu_2^e).$$ 
From the above it follows that ICRM learns $p(x_1,\mu_2^e)$. In comparison, consider standard ERM learns a non-linear model on $(x_1,x_2)$. Consider the DAG corresponding to setting \eqref{eqn: dgp_inv_nonlinear}.  We assume that the joint distribution described in \eqref{eqn: dgp_inv_nonlinear}  is Markov w.r.t to the following DAG $x_1\rightarrow y \leftarrow \mu_2^e \rightarrow x_2$. As a result, $y \not \perp x_2 |x_1$. This follows from the fact there is a path $y$ to $x_2$ through $\mu_2^e$ and is not blocked by $x_1$. From $y \not \perp x_2 |x_1$ it follows that ERM learns a predictor that relies on both $x_1$ and $x_2$.  Therefore, ICRM learns the right invariant model and does not rely on $x_2$ and ERM relies on spurious feature $x_2$.

\subsection{Illustration of failure of existing MTL methods}

In this section, we provide a simple example to show the failure mode of marginal transfer learning (MTL) methods that are based on averaging $\frac{1}{|C|}\sum_{x_i \in C}\Phi(\cdot)$ to summarize information about the environment. These methods can be summarized to take the following form: 

\begin{equation}
    f\bigg(\sum_{x_i\in C} \Phi(x_i), x\bigg)
\end{equation}
We are only going to consider maps $\Phi$ that are differentiable.

\paragraph{Example.} Suppose we want to learn the following function

\begin{equation}
   h(x, C) = \sum_{x_i \in C}I(x<x_{i}),
\end{equation}
where $x_i$ is the $i^{th}$ input in the context and $x$ is the current query. We claim that if $f\bigg(\sum_{x_i\in C} \Phi(x_i), x\bigg) = h(x,C) $ for all $x\in \mathbb{R},C\in \mathbb{R}^{|C|}$, then the output dimension of $\Phi$ grows in context length $|C|$. Suppose this was not the case. If $\Phi's$ output dimension is smaller than $|C|$, then $\Phi$ cannot be a differentiable bijection. As a result, there exists two contexts $C$ and $C'$ of same length for which $\sum_{x_i\in C} \Phi(x_i) = \sum_{x_i\in C'} \Phi(x_i)$. We argue that there exists an $x$ such that $h(x,C)\not=h(x,C')$. This would lead to a contradiction as $f\bigg(\sum_{x_i\in C} \Phi(x_i), x\bigg) = h(x,C) $ for all $x,C$. Without loss of generality, suppose that the smallest value of context $C$ is smaller than that in context $C'$. If $x$ is larger than smallest value of $C$ but lesser than smallest value of $C'$, then $h(x,C') = |C'|$ on the other hand $h(x,C)\leq |C|-1 = |C'|-1$.

\section{Related work}

\paragraph{A brief tour of domain generalization.} \cite{muandet2013domain} developed kernel methods to learn transformations such that the distance between the feature distributions across domains is minimized and the information between the features and the target labels is preserved.  The pioneering work of \cite{dann} proposes a method inspired from generative adversarial networks  to learn feature representations that are similar across domains. \cite{sun2016deep} developed a method based on a natural strategy to match the means and covariances of feature representations across domains. \cite{li2018domain} went a step further to enforce invariance on the distribution of representations conditional on the labels. In a parallel line of work, led by \cite{peters2016causal, rojas2018invariant, irm}, the proposals sought to learn representations such that the distribution of labels conditional on the representation are invariant across domains.
These works were followed by several interesting proposals to enforce invariance -- 
\citep{teney2020unshuffling, krueger2020out, ahuja2020invariant, jin2020enforcing, chang2020invariant,  mahajan2020domain,koyama2020out,muller2020learning,parascandolo2020learning, ahuja2021invariance,  robey2021model, wald2021calibration, chen2022iterative, wang2022provable, zhang2023missing, eastwood2022probable, rame2022fishr, veitch2021counterfactual, makar2022causally} -- which is an incomplete representative list. See \cite{shen2021towards} for a more comprehensive survey of these works. Most of the above works have focused on learning features that enable better generalization. Recently there been an intriguing line of work from \cite{kirichenko2022last,izmailov2022feature} that shifts the focus from feature learning to last layer retraining. These works show that under certain conditions (e.g., avaiability of some data that does not carry spurious correlations) one can carry out last layer retraining and achieve significant out-of-distribution performance improvements. 

In the main body of the paper, we already discussed the other prominent line of work in domain generalization on marginal transfer learning, where the focus is to leverage the distributional features and learn environment specific relationships. This line of work was started by the notable work of \cite{marginal} and has been followed up by several important proposals such as \cite{arm, contextvit}.

\section{Supplementary experimental details and assets disclosure}
\subsection{Assets}
We do not introduce new data in the course of this work. Instead, we use publicly available widely used image datasets for the purposes of benchmarking and comparison.

\subsection{Hardware and setup}
Each experiment was performed on 8 NVIDIA Tesla V100 GPUs with 32GB accelerator RAM for a single training run. The CPUs used were Intel Xeon E5-2698 v4 processors with 20 cores and 384GB RAM. All experiments use the PyTorch deep-learning framework 

\subsection{Datasets}\label{sec:datasets}

\subsubsection{Federated Extended MNIST (FEMNIST)} Building on the Extended MNIST (EMNIST) dataset, which includes images of handwritten uppercase and lowercase alphabets along with digits, FEMNIST enriches this data by attributing each data point to its originating writer. This extension associates each 28$\times$28-sized image in the dataset to one of the 62 classes. In our setup, each writer serves as a distinct environment. We evaluate the performance of each method based on both worst-case and average accuracy across a set of 35 test users, who are distinct from the 262 training users and 50 validation users. Unlabelled data from an environment in this dataset could provide cues about the writing style of the user and disambiguate data points. \\

\subsubsection{Rotated MNIST} We employ a customized version of the MNIST dataset as in ~\cite{arm}. The dataset contains images rotated in increments of 10 degrees, ranging from 0 to 130 degrees. Each degree of rotation constitutes a separate environment, effectively acting as a distinct value. The training set for the two most extreme rotations, 120 and 130 degrees, contains only 108 data points each. For rotations between 90 and 110 degrees, each environment includes 324 data points. The total training set comprises 32,292 points. For evaluation, test images are generated from the MNIST test set, and are duplicated for each environment. Performance metrics include both worst-case and average accuracy across these testing domains. Analogous to FEMNIST, unlabeled samples from an environment within this dataset can assist in distinguishing images that may seem similar due to their rotated orientations.\\

\subsubsection{WILDS Camelyon17} We use the Camelyon17 dataset, part of the WILDS benchmark \citep{koh2021wilds}, which features image patches derived from whole-slide lymph node sections of patients with potential metastatic breast cancer. Each patch is labeled to indicate the presence or absence of a tumor. In our experimental design, each participating hospital is treated as a distinct environment. The dataset is partitioned in alignment with the official WILDS configuration: three hospitals contribute to the training set, a fourth is designated for validation, and the remaining hospital's data is used for testing. \\

\subsubsection{Tiny ImageNet-C} Adapting the methodology from \cite{hendrycks2019benchmarking}, we introduce 56 distinct distortions to the training set, treating each as a separate environment. For evaluation, we use a non-overlapping set of 22 test distortions, largely differing in nature from those used in training. Each 64X64-sized distorted image is associated with one of the 200 classes in the dataset. This setup permits an investigation into whether exposure to distortions during training equips the model to better manage novel distortions during testing. We assess performance through both worst-case and average accuracies across these test distortions.

\subsection{Experimental protocols}\label{sec: experimental setup}
To ensure a fair comparison across different algorithms for each dataset, we use a standardized neural network backbone. The details for these architectures are provided in~\Cref{table:architectures} and \Cref{table:mnist_convnet}. We use the ConvNet architecture as outlined in \cite{arm}. 
\par For \method{}, the same backbone is used to featurize the input, which is then processed by the decoder-only Transformer \citep{attention} architecture from the GPT-2 Transformer family \citep{radford2019language}. Our model is standardized to have 12 layers, 4 attention heads, and a 128-dimensional embedding space across all datasets. Linear layers are employed to map both the input sequence to the transformer's latent embedding and the model's predicted output vector to the output label. For training \method{} on larger datasets like WILDS Camelyon17 and Tiny ImageNet-C, we start with a ResNet50 model pre-trained on ImageNet (as shown in~\Cref{table:architectures}) and freeze all batch normalization layers before fine-tuning. 

\par We adopt the same Context Network as used in ARM, specifically retaining their choice of output channels -- one for smaller datasets like FEMNIST and Rotated MNIST, and three for the others.

\par For TENT, all reported metrics are based on its episodic version, where the model is reset to its trained state after processing each batch. This ensures a fair comparison with other methods. Additionally, during testing, the model's parameters are updated for 10 steps using stochastic gradient descent by minimization test entropy across all datasets.

\vskip 0.5cm
\noindent
\begin{minipage}{0.55\textwidth}
    \captionof{table}{Network architectures for each dataset.}
    \begin{tabular}{ll|l}
    \toprule
    \multirow{2}{*}{\textbf{Dataset}} & \multicolumn{2}{c}{Architecture} \\
    \cmidrule(l){2-3}
     & \method{} & Others\\
     \midrule
    FEMNIST & \multirow{2}{*}{\parbox{3cm}{ConvNet + GPT2 Transformer}} & \multirow{2}{*}{ConvNet} \\
    Rotated MNIST & \\
    \midrule
    Camelyon17 &  \multirow{2}{*}{\parbox{3cm}{ResNet-50 + GPT2 Transformer}} & \multirow{2}{*}{ResNet-50} \\ 
    Tiny ImageNet-C & \\
    \bottomrule
    \end{tabular}
\label{table:architectures}
\end{minipage}
\hfill
\begin{minipage}{0.35\textwidth}
    \captionof{table}{ConvNet architecture for \citep{arm}. We use 2\(\times\)2 kernels and ``same'' padding.}
    \begin{tabular}{ll}
    \toprule
    \textbf{\#} & \textbf{Layer}\\
    \midrule
        1  & Conv2D (in=\(d\), out=128)\\
        2  & BatchNorm2d (dim=129)\\
        3  & ReLU \\
        4  & Max Pooling (2) \\
        5  & Conv2D (in=128, out=128)\\
        6  & BatchNorm2d (dim=128)\\
        7  & ReLU \\
        8  & Max Pooling (2) \\
        9  & Global average-pooling\\
    \bottomrule
    \end{tabular}
\label{table:mnist_convnet}
\end{minipage}

\vskip 0.5cm

We list all hyperparameters, their default settings, and search boundaries for random sweeps in~\Cref{table:hyperparameters}. The maximum context length, or support, is fixed at 100 for all algorithms. All models are optimized using the Adam optimizer \citep{kingma2014adam}. To ensure a fair comparison, we perform a random search of 5 trials across the hyperparameter range (refer to~\Cref{table:hyperparameters}) for each algorithm. The model with the highest validation set accuracy is selected for each run. We then report the average of this number across three independent runs of the entire sweep, and its corresponding standard error.

\begin{table}[!htb]
    \caption{Hyperparameters, their default values and distributions for random search.} 
    \begin{center}
    { 
    \begin{tabular}{llll}
        \toprule
        \textbf{Condition} & \textbf{Parameter} & \textbf{Default value} & \textbf{Random distribution}\\
        \midrule
        \multirow{2}{*}{ResNet}       & learning rate & 0.0001 & $10^{\text{Uniform}(-5, -3.5)}$\\
                                      & weight decay & 0    & $10^{\text{Uniform}(-6, -2)}$\\
        \midrule
        \multirow{2}{*}{not ResNet}   & learning rate & 0.0001 & $10^{\text{Uniform}(-4.5, -2.5)}$\\
                                      & weight decay & 0    & $10^{\text{Uniform}(-6, -2)}$\\
        \bottomrule
    \end{tabular}
    }
    \end{center}
    \label{table:hyperparameters}
\end{table}

\clearpage
\newpage
\section{Additional experiments}
\subsection{Adaptation curves of various algorithms}
\begin{figure}[!htb]
\centering

\begin{minipage}{\textwidth}
\centering
\includegraphics[width=.39\textwidth]{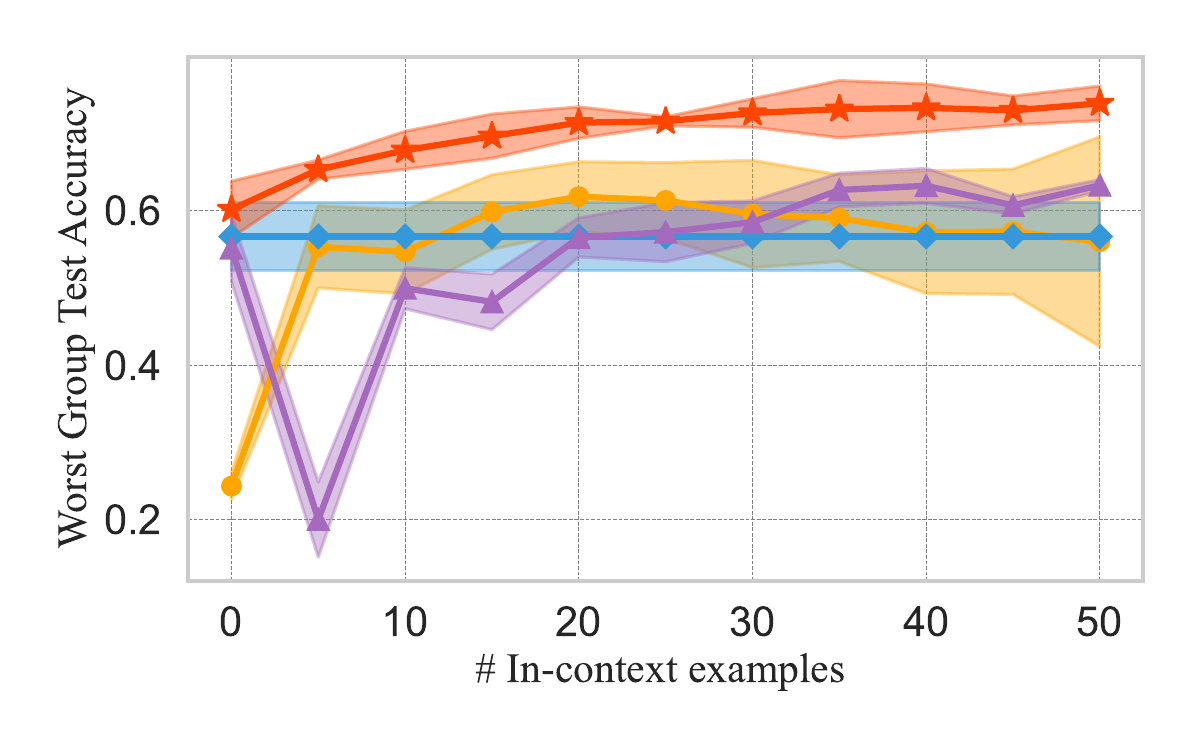}
\includegraphics[width=.39\textwidth]{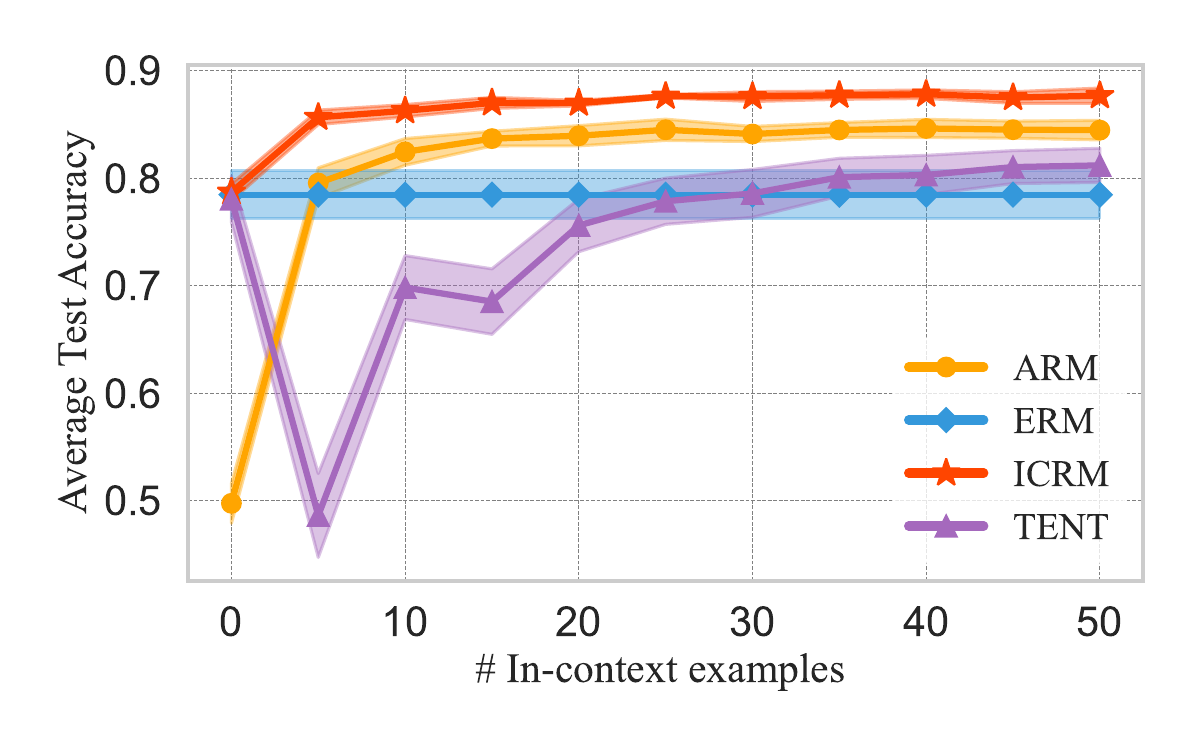}
\end{minipage}
\hfill

\begin{minipage}{\textwidth}
\centering
\includegraphics[width=.39\textwidth]{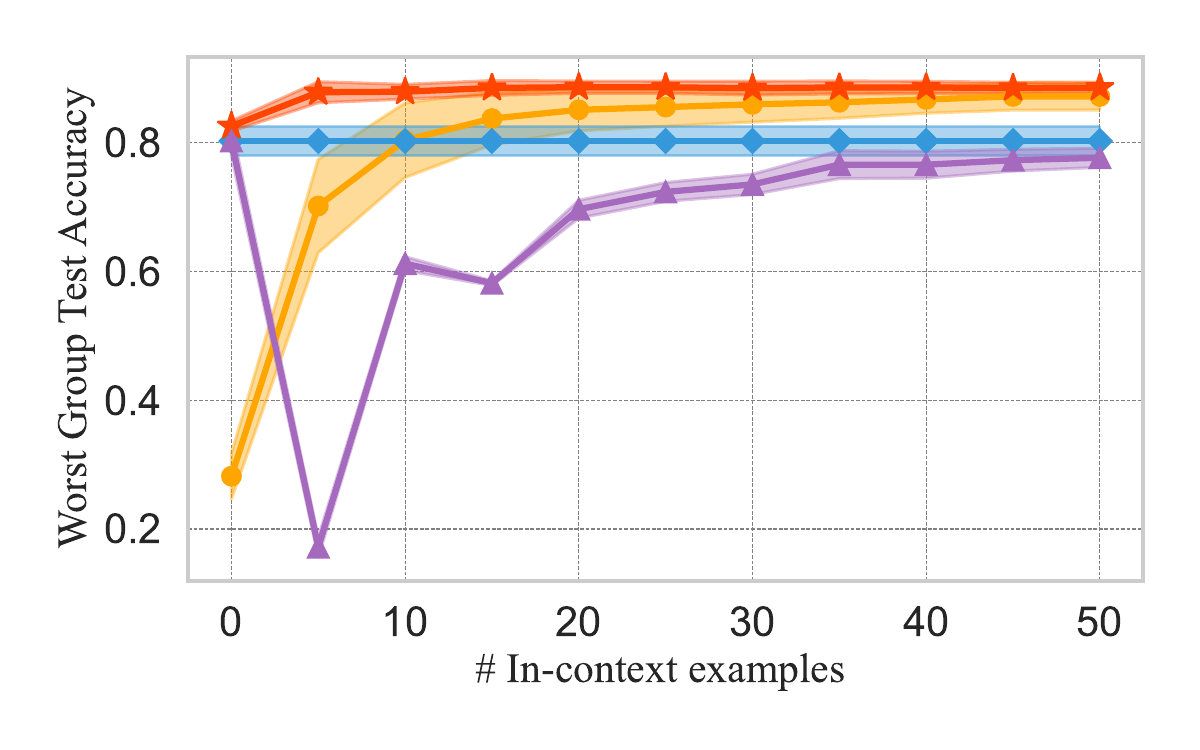}
\includegraphics[width=.39\textwidth]{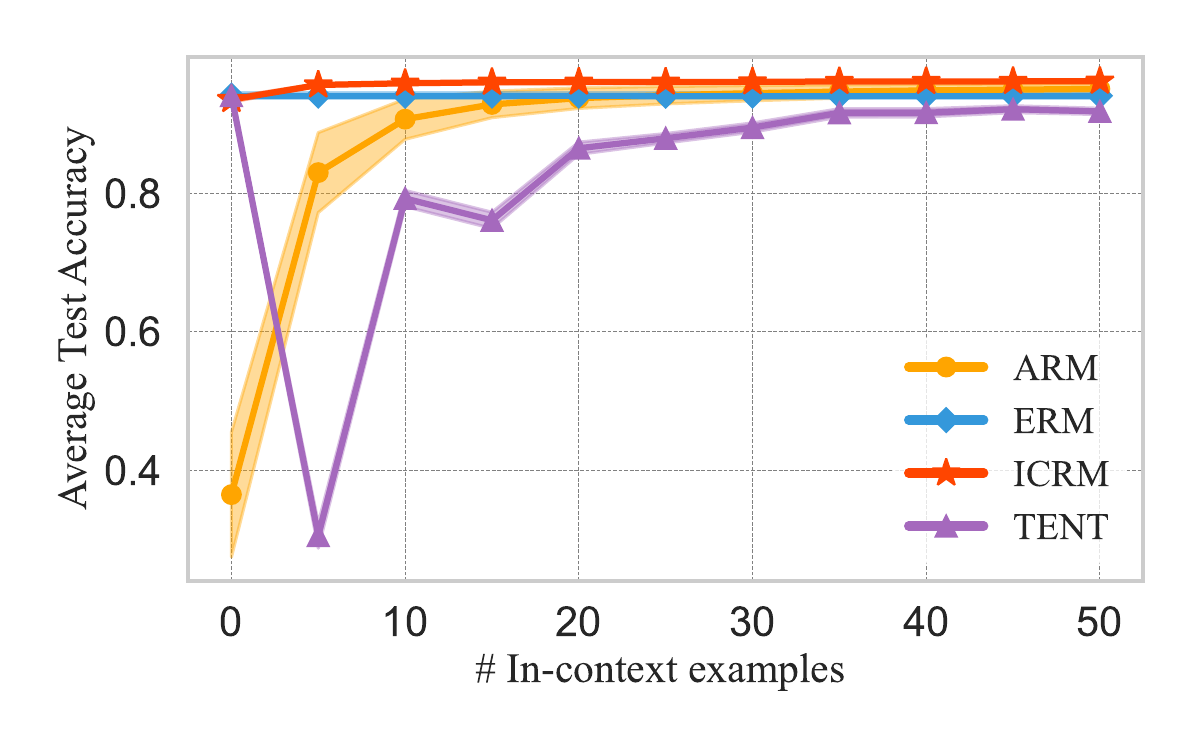}
\end{minipage}
\hfill

\begin{minipage}{\textwidth}
\centering
\includegraphics[width=.39\textwidth]{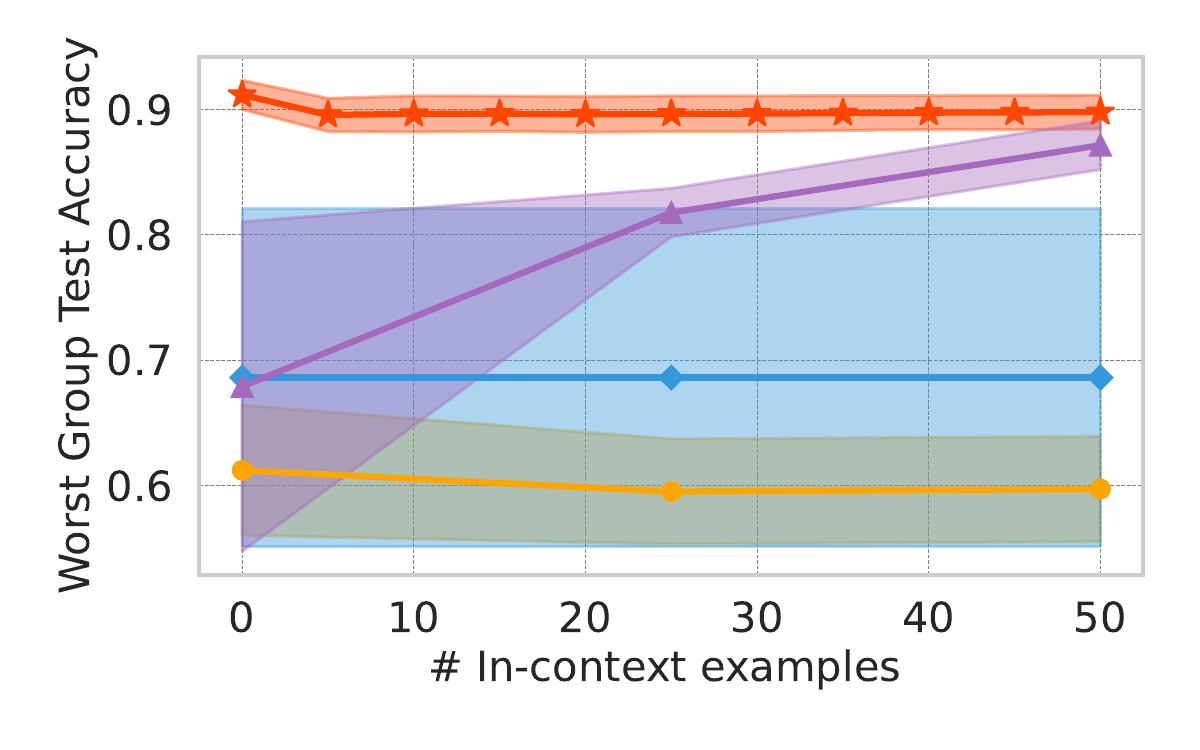}
\includegraphics[width=.39\textwidth]{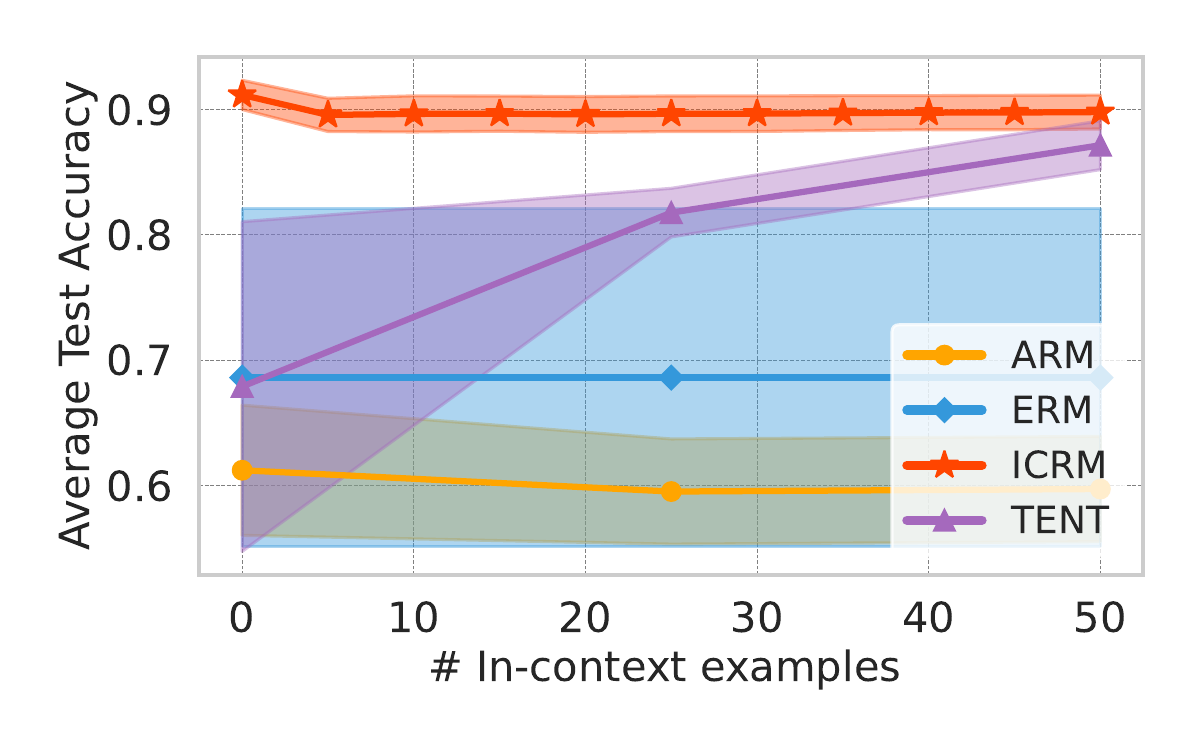}
\end{minipage}
\hfill

\begin{minipage}{\textwidth}
\centering
\includegraphics[width=.39\textwidth]{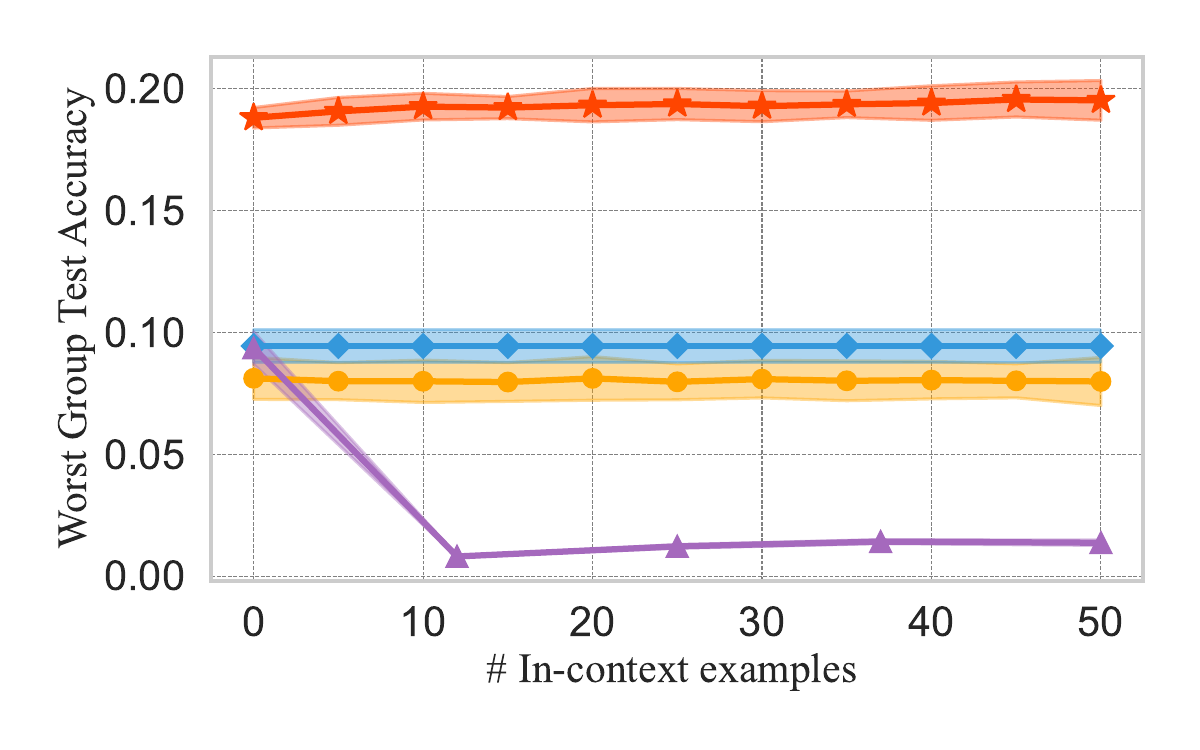}
\includegraphics[width=.39\textwidth]{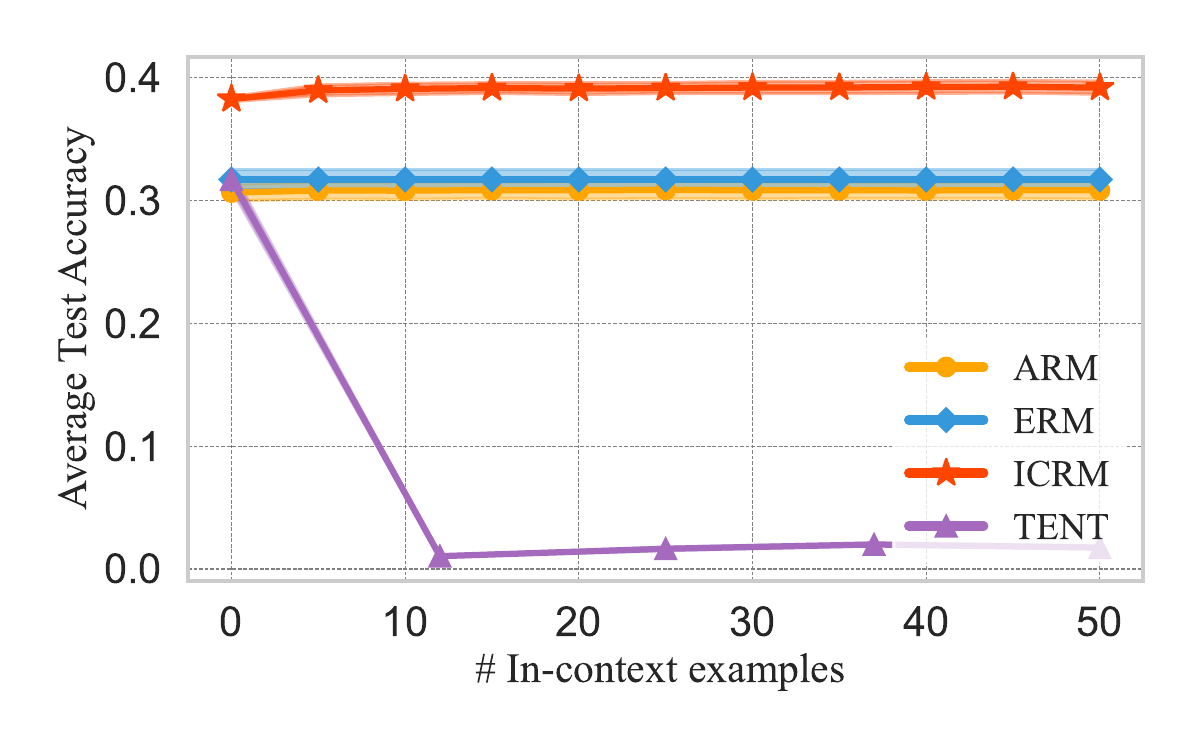}
\end{minipage}
\hfill
\caption{
  Accuracy adaptation curves for worst accuracy (left) and average accuracy (right) across the test environment as a function of increasing count of context samples.
  Showing results in order for FEMNIST(top), RotatedMNIST, WILDS Camelyon17 and Tiny ImageNet-C(bottom).
  The average and worst-case accuracy plots for WILDS Camelyon17 are identical since the dataset contains only a single test environment.
  }\label{fig: mainplots}
\end{figure}

\clearpage
\newpage
\subsection{Domain generalization accuracies per algorithm and dataset}
\begin{table}[htb!]
    \caption{Average out-of-distribution test accuracies along with their corresponding standard errors for various counts of context samples. The methods compared include Adaptive Risk Minimization (ARM), Empirical Risk Minimization (ERM), Test Entropy Minimization (TENT), and our method \method{} on FEMNIST, Rotated MNIST, WILDS Camelyon17 and  Tiny-ImageNet-C.}
    \begin{center}
    \resizebox{0.9\textwidth}{!}{%
    \begin{tabular}{lccccc}
        \toprule
        \textbf{Dataset / algorithm} & \multicolumn{5}{c}{\textbf{Average test accuracy (by \# in-context examples)}} \\
        \coloredMidrule{white}{alternateRowColor}
        \rowcolor{alternateRowColor}
        FEMNIST & 0 & 25 & 50 & 75 & 100  \\
        \coloredBelowRuleSep{white}
        \quad ARM       &  49.5 $\pm$ 1.0 &  83.9 $\pm$ 0.5 &  84.4 $\pm$ 0.5 & 84.7 $\pm$ 0.6& 84.6 $\pm$ 0.3  \\
        \quad TENT  & 78.1 $\pm$ 1.2 & 77.9$\pm$ 1.2 & 81.2 $\pm$ 0.9 & 82.5$\pm$ 0.9 & 83.3 $\pm$ 0.8 \\
        \quad ERM       &  79.3  $\pm$ 0.4  & 79.3 $\pm$ 0.4  & 79.3 $\pm$ 0.4  & 79.3 $\pm$ 0.4 & 79.3 $\pm$ 0.4 \\
        \quad ICRM       &  78.7  $\pm$ 0.5 & 87.2 $\pm$ 0.4 & 87.4 $\pm$ 0.5  & 87.5 $\pm$ 0.2 &  87.8 $\pm$ 0.2 \\
        \coloredMidrule{white}{alternateRowColor}
        \rowcolor{alternateRowColor}
        Rotated MNIST & 0 & 25 & 50 & 75 & 100 \\
        \coloredBelowRuleSep{white}
        \quad ARM       &  36.5 $\pm$ 5.2 & 94.2  $\pm$ 0.7 &  95.1 $\pm$ 0.4 & 95.3  $\pm$ 0.4& 95.5  $\pm$  0.3\\
        \quad TENT & 94.1 $\pm$ 0.3 & 88.0 $\pm$ 0.4 & 91.9 $\pm$ 0.3 & 93.8 $\pm$ 0.2 & 94.3 $\pm$ 0.2 \\
        \quad ERM       & 94.2 $\pm$ 0.3 & 94.2  $\pm$ 0.3& 94.2 $\pm$ 0.3  & 94.2 $\pm$ 0.3 & 94.2 $\pm$ 0.3 \\
        \quad ICRM       &  93.6 $\pm$ 0.2 &   96.1 $\pm$ 0.1 &  96.2 $\pm$ 0.1 & 96.2 $\pm$ 0.1&96.2 $\pm$ 0.1\\
        \coloredMidrule{white}{alternateRowColor}
        \rowcolor{alternateRowColor}
        WILDS Camelyon17 & 0 & 25 & 50 & 75 & 100 \\
        \coloredBelowRuleSep{white}
        \quad ARM       &  61.2 $\pm$ 5.2 &  59.5 $\pm$ 4.2 & 59.7 $\pm$ 4.2  & 59.7 $\pm$ 4.3 & 59.7 $\pm$ 4.2 \\
        \quad TENT       & 67.9 $\pm$ 7.6 & 81.8 $\pm$ 1.1 & 87.2 $\pm$ 1.1 & 89.4  $\pm$ 1.1& 89.4 $\pm$ 1.0   \\
        \quad ERM       &  68.6  $\pm$ 7.8 & 68.6 $\pm$ 7.8  &  68.6  $\pm$ 7.8& 68.6  $\pm$ 7.8& 68.6  $\pm$ 7.8  \\
        \quad ICRM       &  92.0 $\pm$ 0.6 & 90.7  $\pm$ 0.8 & 90.8 $\pm$ 0.8  & 90.8 $\pm$ 0.8 & 90.8  $\pm$ 0.8\\
        \coloredMidrule{white}{alternateRowColor}
        \rowcolor{alternateRowColor}
        Tiny ImageNet-C & 0 & 25 & 50 & 75 & 100 \\
        \coloredBelowRuleSep{white}
        \quad ARM       &  30.8 $\pm$ 0.2 &  31.0 $\pm$ 0.2 &  31.0 $\pm$ 0.2 &  31.0 $\pm$ 0.2 & 31.0 $\pm$ 0.2  \\
        \quad TENT       & 31.7 $\pm$ 0.5 & 1.6 $\pm$ 0.1 & 1.7 $\pm$ 0.1 & 2.0 $\pm$ 0.1 & 2.1 $\pm$ 0.1 \\
        \quad ERM       &  31.8 $\pm$ 0.6 &  31.8 $\pm$ 0.6 & 31.8  $\pm$ 0.6 & 31.8 $\pm$ 0.6 & 31.8 $\pm$ 0.6\\
        \quad ICRM       &  38.3 $\pm$ 0.1 & 39.2  $\pm$ 0.3 & 39.2 $\pm$ 0.3  & 39.2 $\pm$ 0.3 & 39.2 $\pm$ 0.3\\
        \bottomrule
    \end{tabular}
    }
    \end{center}
    \label{table:supp_maintable_avg}
\end{table}

\begin{table}[htb!]
    \caption{Worst environment out-of-distribution test accuracies along with their corresponding standard errors for various counts of context samples. The methods compared include Adaptive Risk Minimization (ARM), Empirical Risk Minimization (ERM), Test Entropy Minimization (TENT), and our method \method{} on FEMNIST, Rotated MNIST, WILDS Camelyon17 and  Tiny-ImageNet-C.}
    \begin{center}
    \resizebox{0.9\textwidth}{!}{%
    \begin{tabular}{lccccc}
        \toprule
        \textbf{Dataset / algorithm} & \multicolumn{5}{c}{\textbf{Worst case test accuracy (by \# in-context examples)}} \\
        \coloredMidrule{white}{alternateRowColor}
        \rowcolor{alternateRowColor}
        FEMNIST &  0 & 25 & 50 & 75 & 100  \\
        \coloredBelowRuleSep{white}
        \quad ARM  & 23.6 $\pm$ 1.7  &  59.5 $\pm$ 3.5 & 60.7 $\pm$ 3.8  & 57.0 $\pm$ 7.3 & 58.8 $\pm$ 4.0 \\
        \quad TENT &55.2 $\pm$ 2.5 & 57.2 $\pm$ 2.2 & 63.3 $\pm$ 0.4 & 65.9  $\pm$ 0.6& 67.2 $\pm$ 1.0\\
        \quad ERM  &59.0  $\pm$ 0.2 & 59.0  $\pm$  0.2& 59.0  $\pm$ 0.2 & 59.0 $\pm$ 0.2  &  59.0 $\pm$ 0.2 \\
        \quad ICRM &  59.8 $\pm$ 0.7 & 69.3 $\pm$ 0.0  & 70.6  $\pm$ 2.3 & 70.6 $\pm$ 1.5 & 70.6 $\pm$ 0.7 \\
        \coloredMidrule{white}{alternateRowColor}
        \rowcolor{alternateRowColor}
        Rotated MNIST & 0 & 25 & 50 & 75 & 100 \\
        \coloredBelowRuleSep{white}
        \quad ARM & 28.2  $\pm$ 2.1 & 85.3  $\pm$ 1.6 & 87.2  $\pm$ 1.0 & 87.9 $\pm$ 1.0 & 87.9 $\pm$ 0.9 \\
        \quad TENT & 80.2 $\pm$ 1.3 & 88.5 $\pm$ 0.8 & 88.5 $\pm$ 0.9 & 80.2 $\pm$ 1.0 & 81.3 $\pm$ 1.0\\
        \quad ERM  &  80.8 $\pm$ 1.1 & 80.8 $\pm$ 1.1  & 80.8 $\pm$ 1.1  & 80.8 $\pm$ 1.1 & 80.8 $\pm$ 1.1  \\
        \quad ICRM &  82.5  $\pm$ 0.5 &  88.5 $\pm$ 0.5 &  88.5 $\pm$ 0.5 & 88.8 $\pm$ 0.5 & 88.8 $\pm$ 0.4 \\
        \coloredMidrule{white}{alternateRowColor}
        \rowcolor{alternateRowColor}
        WILDS Camelyon17 & 0 & 25 & 50 & 75 & 100 \\
        \coloredBelowRuleSep{white}
        \quad ARM       &  61.2 $\pm$ 5.2 &  59.5 $\pm$ 4.2 & 59.7 $\pm$ 4.2  & 59.7 $\pm$ 4.3 & 59.7 $\pm$ 4.2 \\
        \quad TENT       & 67.9 $\pm$ 7.6 & 81.8 $\pm$ 1.1 & 87.2 $\pm$ 1.1 & 89.4  $\pm$ 1.1& 89.4 $\pm$ 1.0   \\
        \quad ERM       &  68.6  $\pm$ 7.8 & 68.6 $\pm$ 7.8  &  68.6  $\pm$ 7.8& 68.6  $\pm$ 7.8& 68.6  $\pm$ 7.8  \\
        \quad ICRM       &  92.0 $\pm$ 0.6 & 90.7  $\pm$ 0.8 & 90.8 $\pm$ 0.8  & 90.8 $\pm$ 0.8 & 90.8  $\pm$ 0.8\\
        \coloredMidrule{white}{alternateRowColor}
        \rowcolor{alternateRowColor}
        Tiny ImageNet-C & 0 & 25 & 50 & 75 & 100 \\
        \coloredBelowRuleSep{white}
        \quad ARM  &  8.2 $\pm$ 0,3  & 8.3 $\pm$ 0.3  &  8.2 $\pm$ 0.3 & 8.3 $\pm$ 0.3 & 8.2 $\pm$ 0.3 \\
        \quad TENT & 1.2 $\pm$0.4   & 1.4 $\pm$ 0.0 & 1.6 $\pm$ 0.1 & 1.6  $\pm$  0.0 & 1.6  $\pm$ 0.0\\
        \quad ERM  & 9.5  $\pm$ 0.4 &  9.5 $\pm$ 0.4 &  9.5 $\pm$ 0.4 &  9.5 $\pm$ 0.4 & 9.5 $\pm$ 0.4  \\
        \quad ICRM  & 18.8 $\pm$ 0.2  &  19.2 $\pm$ 0.1 & 19.5  $\pm$ 0.2 & 19.5 $\pm$ 0.1 & 19.4  $\pm$ 0.2\\
        \bottomrule
    \end{tabular}
    }
    \end{center}
    \label{table:supp_maintable_wo}
\end{table}

\begin{table}[htb!]
    \caption{Average out-of-distribution test accuracies along with their corresponding standard errors for \method{} and ICRM-Mix across FEMNIST, Rotated MNIST, WILDS Camelyon17 and  Tiny-ImageNet-C. ICRM-Mix trains on sequences with samples drawn i.i.d. from the unified dataset comprising various environments.}    
    \begin{center}
    \resizebox{0.9\textwidth}{!}{%
    \begin{tabular}{lccccc}
        \toprule
        \textbf{Dataset / algorithm} & \multicolumn{5}{c}{\textbf{Average test accuracy (by \# in-context examples)}} \\
        \coloredMidrule{white}{alternateRowColor}
        \rowcolor{alternateRowColor}
        FEMNIST & 0 & 25 & 50 & 75 & 100  \\
        \coloredBelowRuleSep{white}
        
        \quad ICRM       &  78.7  $\pm$ 0.5  & 87.2  $\pm$ 0.4& 87.4  $\pm$ 0.5 & 87.5 $\pm$ 0.2 &  87.8 $\pm$ 0.2  \\
        \quad ICRM-Mix  & 77.6  $\pm$ 0.8 & 81.1 $\pm$ 0.2 & 81.1 $\pm$ 0.2 & 80.9 $\pm$ 0.3 &  80.9 $\pm$ 0.1 \\
        \coloredMidrule{white}{alternateRowColor}
        \rowcolor{alternateRowColor}
        Rotated MNIST & 0 & 25 & 50 & 75 & 100 \\
        \coloredBelowRuleSep{white}
        \quad ICRM       &  93.6 $\pm$ 0.2  &   96.1 $\pm$ 0.1 &  96.2 $\pm$ 0.1 & 96.2 $\pm$ 0.1& 96.2 $\pm$ 0.1 \\
        \quad ICRM-Mix   &  88.9  $\pm$ 1.4& 92.6 $\pm$ 0.3 & 92.7 $\pm$ 0.2 & 92.6  $\pm$ 0.3& 92.7 $\pm$ 0.2 \\
        \coloredMidrule{white}{alternateRowColor}
        \rowcolor{alternateRowColor}
        WILDS Camelyon17 & 0 & 25 & 50 & 75 & 100 \\
        \coloredBelowRuleSep{white}
        \quad ICRM       &  92.0  $\pm$ 0.6& 90.7  $\pm$ 0.8 & 90.8  $\pm$ 0.8 & 90.8  $\pm$ 0.8& 90.8 $\pm$ 0.8 \\
        \quad ICRM-Mix   & 92.9  $\pm$ 0.3 & 90.7 $\pm$ 0.6 & 90.8  $\pm$ 0.5 & 90.7  $\pm$ 0.5& 90.7 $\pm$ 0.5 \\
        \coloredMidrule{white}{alternateRowColor}
        \rowcolor{alternateRowColor}
        Tiny ImageNet-C & 0 & 25 & 50 & 75 & 100 \\
        \coloredBelowRuleSep{white}
        \quad ICRM       & 38.3 $\pm$ 0.1 & 39.2 $\pm$ 0.3  & 39.2 $\pm$ 0.3  & 39.2  $\pm$ 0.3& 39.2 $\pm$ 0.3 \\
        \quad ICRM-Mix   & 38.4  $\pm$ 0.2& 39.3 $\pm$ 0.2  & 39.3 $\pm$ 0.2 & 39.3  $\pm$ 0.2& 39.3  $\pm$ 0.2 \\
        \bottomrule
    \end{tabular}
    }
    \end{center}
    \label{table:sup_ablation_iid_icl_avg}
\end{table}

\begin{table}[htb!]
    \caption{Worst environment out-of-distribution test accuracies along with their corresponding standard errors for \method{} and ICRM-Mix across FEMNIST, Rotated MNIST, WILDS Camelyon17 and  Tiny-ImageNet-C. ICRM-Mix trains on sequences with samples drawn i.i.d. from the unified dataset comprising various environments.}    \begin{center}
    \resizebox{0.9\textwidth}{!}{%
    \begin{tabular}{lccccc}
        \toprule
        \textbf{Dataset / algorithm} & \multicolumn{5}{c}{\textbf{Worst case test accuracy (by \# in-context examples)}} \\
        \coloredMidrule{white}{alternateRowColor}
        \rowcolor{alternateRowColor}
        FEMNIST & 0 & 25 & 50 & 75 & 100  \\
        \coloredBelowRuleSep{white}
        
        \quad ICRM  & 59.8 $\pm$ 0.7 & 69.3  $\pm$ 0.0 & 70.6  $\pm$ 2.3 & 70.6 $\pm$ 1.5 & 70.6 $\pm$ 0.7 \\
        \quad ICRM-Mix & 57.5  $\pm$ 1.4& 62.7 $\pm$ 1.1 & 65.0 $\pm$ 0.3 & 64.1 $\pm$ 1.5 & 62.9 $\pm$ 2.3 \\
        \coloredMidrule{white}{alternateRowColor}
        \rowcolor{alternateRowColor}
        Rotated MNIST & 0 & 25 & 50 & 75 & 100  \\
        \coloredBelowRuleSep{white}
        \quad ICRM    &  82.5 $\pm$ 0.5 &  88.5  $\pm$ 0.5&  88.5 $\pm$ 0.5 & 88.8  $\pm$ 0.5& 88.8 $\pm$ 0.4 \\
        \quad ICRM-Mix   & 68.8 $\pm$ 3.8 & 77.1 $\pm$ 0.7 & 76.8 $\pm$ 0.9 & 76.4 $\pm$ 0.9 & 76.6 $\pm$ 0.9  \\
        \coloredMidrule{white}{alternateRowColor}
        \rowcolor{alternateRowColor}
        WILDS Camelyon17 & 0 & 25 & 50 & 75 & 100 \\
        \coloredBelowRuleSep{white}
        \quad ICRM       &  92.0  $\pm$ 0.6& 90.7  $\pm$ 0.8 & 90.8  $\pm$ 0.8 & 90.8  $\pm$ 0.8& 90.8 $\pm$ 0.8 \\
        \quad ICRM-Mix   & 92.9  $\pm$ 0.3 & 90.7 $\pm$ 0.6 & 90.8  $\pm$ 0.5 & 90.7  $\pm$ 0.5& 90.7 $\pm$ 0.5 \\
        \coloredMidrule{white}{alternateRowColor}
        \rowcolor{alternateRowColor}
        Tiny ImageNet-C & 0 & 25 & 50 & 75 & 100 \\
        \coloredBelowRuleSep{white}
        \quad ICRM   & 18.8  $\pm$ 0.2 &  19.2 $\pm$ 0.1 & 19.5 $\pm$ 0.2  & 19.5 $\pm$ 0.1 & 19.4 $\pm$ 0.2 \\
        \quad ICRM-Mix & 18.7 $\pm$ 0.2 & 19.2 $\pm$ 0.2 & 19.4 $\pm$ 0.1 & 19.5 $\pm$ 0.1 & 19.4  $\pm$ 0.1\\
        \bottomrule
    \end{tabular}
    }
    \end{center}
    \label{table:sup_ablation_iid_icl_avg}
\end{table}

\begin{table}[htb!]
    \caption{Average out-of-distribution test accuracies along with their corresponding standard errors for ARM$^{+}$ and ERM$^{+}$ in contrast to their base algorithms, ARM and ERM across FEMNIST, Rotated MNIST, WILDS Camelyon17 and  Tiny-ImageNet-C. }    
    \begin{center}
    \resizebox{0.9\textwidth}{!}{%
    \begin{tabular}{lccccc}
        \toprule
        \textbf{Dataset / algorithm} & \multicolumn{5}{c}{\textbf{Average test accuracy (by \# in-context examples)}} \\ \midrule
        \rowcolor{alternateRowColor}
        FEMNIST & 0 & 25 & 50 & 75 & 100 \\
        \coloredBelowRuleSep{white}
        \quad ARM       &  49.5 $\pm$ 1.0 &  83.9 $\pm$ 0.5 &  84.4 $\pm$ 0.5 & 84.7  $\pm$ 0.6& 84.6 $\pm$ 0.3  \\
        \quad ARM$^{+}$ & 71.4 $\pm$ 1.2 & 83.4 $\pm$ 0.2 &  84.0 $\pm$ 0.2 & 83.8 $\pm$ 0.2 & 83.5 $\pm$ 0.1 \\
        \midrule
        \quad ERM  &  79.3  $\pm$ 0.4 & 79.3 $\pm$ 0.4  & 79.3 $\pm$ 0.4  & 79.3 $\pm$ 0.4 & 79.3 $\pm$ 0.4 \\
        \quad ERM$^{+}$ & 77.4  $\pm$ 1.3& 77.4 $\pm$ 1.3 & 77.4 $\pm$ 1.3 & 77.4 $\pm$ 1.3 & 77.4 $\pm$ 1.3 \\
        \coloredMidrule{white}{alternateRowColor}
        \rowcolor{alternateRowColor}
        Rotated MNIST & 0 & 25 & 50 & 75 & 100 \\
        \coloredBelowRuleSep{white}
        \quad ARM       &  36.5 $\pm$ 5.2 & 94.2 $\pm$ 0.7  &  95.1 $\pm$ 0.4 & 95.3 $\pm$ 0.4 & 95.5 $\pm$ 0.3\\
        \quad ARM$^{+}$   & 86.9  $\pm$ 2.0& 92.6 $\pm$ 0.7 & 92.7  $\pm$ 0.6& 92.8  $\pm$ 0.6& 92.8 $\pm$ 0.6 \\
        \midrule
        \quad ERM       & 94.2 $\pm$ 0.3 & 94.2 $\pm$ 0.3 & 94.2 $\pm$ 0.3  & 94.2 $\pm$ 0.3 & 94.2 $\pm$ 0.3\\
        \quad ERM$^{+}$ & 94.3 $\pm$ 0.4 & 94.3 $\pm$ 0.4 & 94.3 $\pm$ 0.4  & 94.3 $\pm$ 0.4 & 94.3 $\pm$ 0.4   \\
        \coloredMidrule{white}{alternateRowColor}
        \rowcolor{alternateRowColor}
        WILDS Camelyon17 & 0 & 25 & 50 & 75 & 100 \\
        \coloredBelowRuleSep{white}
         \quad ARM       &  61.2 $\pm$ 5.2 &  59.5 $\pm$ 4.2 & 59.7 $\pm$ 4.2  & 59.7 $\pm$ 4.3 & 59.7 $\pm$ 4.2 \\
        \quad ARM$^{+}$ & 55.8  $\pm$ 0.8& 55.1 $\pm$ 1.7 & 55.0  $\pm$ 1.7& 55.0 $\pm$ 1.8 & 55.0 $\pm$ 1.8 \\
        \midrule
        \quad ERM       &  68.6   $\pm$ 7.8& 68.6 $\pm$ 7.8  &  68.6 $\pm$ 7.8 & 68.6 $\pm$ 7.8 & 68.6 $\pm$ 7.8 \\
        \quad ERM$^{+}$ & 50.1 $\pm$ 0.1 & 50.1 $\pm$ 0.1 & 50.1 $\pm$ 0.1 & 50.1  $\pm$ 0.1& 50.1 $\pm$ 0.1 \\
        \coloredMidrule{white}{alternateRowColor}
        \rowcolor{alternateRowColor}
        Tiny ImageNet-C & 0 & 25 & 50 & 75 & 100 \\
        \coloredBelowRuleSep{white}
        \quad ARM       &  30.8  $\pm$ 0.2&  31.0 $\pm$ 0.2 &  31.0 $\pm$ 0.2 &  31.0 $\pm$ 0.2 & 31.0 $\pm$ 0.2 \\
        \quad ARM$^{+}$ & 5.5 $\pm$ 0.2 &5.7  $\pm$ 0.2 & 5.7 $\pm$ 0.2 & 5.7  $\pm$ 0.2& 5.7 $\pm$ 0.2 \\
        \midrule
        \quad ERM       &  31.8 $\pm$ 0.6 &  31.8 $\pm$ 0.6 & 31.8  $\pm$ 0.6 & 31.8  $\pm$ 0.6& 31.8 $\pm$ 0.6\\
        \quad ERM$^{+}$ & 29.7  $\pm$ 0.3& 29.7 $\pm$ 0.3 & 29.7 $\pm$ 0.3 & 29.7 $\pm$ 0.3 & 29.7 $\pm$ 0.3\\
        \bottomrule
    \end{tabular}
    }
    \end{center}
    \label{table:sup_ablation_pluses_avg}
\end{table}

\begin{table}[htb!]
    \caption{Worst environment out-of-distribution test accuracies along with their corresponding standard errors for ARM$^{+}$ and ERM$^{+}$ in contrast to their base algorithms, ARM and ERM across FEMNIST, Rotated MNIST, WILDS Camelyon17 and  Tiny-ImageNet-C. }  
    \begin{center}
    \resizebox{0.9\textwidth}{!}{%
    \begin{tabular}{lccccc}
        \toprule
        \textbf{Dataset / algorithm} & \multicolumn{5}{c}{\textbf{Worst case test accuracy (by \# in-context examples)}} \\
        \coloredMidrule{white}{alternateRowColor}
        \rowcolor{alternateRowColor}
        FEMNIST & 0 & 25 & 50 & 75 & 100 \\
        \coloredBelowRuleSep{white}
        \quad ARM & 23.6  $\pm$ 1.7 &  59.5 $\pm$ 3.5 & 60.7 $\pm$ 3.8  & 57.0 $\pm$ 7.3  & 58.8 $\pm$ 4.0 \\
        \quad ARM$^{+}$ & 51.7 $\pm$ 2.2 & 63.0 $\pm$ 2.1 & 64.0 $\pm$ 0.8 & 60.7 $\pm$ 1.6 & 62.0 $\pm$ 0.8 \\
        \midrule
        \quad ERM  &  59.0 $\pm$ 0.2  & 59.0 $\pm$ 0.2 & 59.0 $\pm$ 0.2 & 59.0 $\pm$ 0.2 &  59.0 $\pm$ 0.2\\
        \quad ERM$^{+}$ & 53.3 $\pm$ 2.7 & 53.3 $\pm$ 2.7  & 53.3 $\pm$ 2.7 & 53.3 $\pm$ 2.7 & 53.3 $\pm$ 2.7 \\
        \coloredMidrule{white}{alternateRowColor}
        \rowcolor{alternateRowColor}
        Rotated MNIST & 0 & 25 & 50 & 75 & 100 \\
        \coloredBelowRuleSep{white}
        \quad ARM  &  28.2 $\pm$ 2.1 & 85.3 $\pm$ 1.6  & 87.2 $\pm$ 1.0  & 87.9 $\pm$ 1.0 & 87.9 $\pm$ 0.9 \\
        \quad ARM$^{+}$ & 71.4  $\pm$ 2.6&  80.9 $\pm$ 1.8 & 81.0 $\pm$ 1.8  & 81.2 $\pm$ 1.9 & 81.1 $\pm$ 1.8\\
        \midrule
        \quad ERM  &  80.8  $\pm$ 1.1& 80.8 $\pm$ 1.1  & 80.8 $\pm$ 1.1  & 80.8 $\pm$ 1.1 & 80.8 $\pm$ 1.1  \\
        \quad ERM$^{+}$ &  81.9 $\pm$ 0.7 & 81.9  $\pm$ 0.7 & 81.9  $\pm$ 0.7 & 81.9 $\pm$ 0.7 & 81.9 $\pm$ 0.7  \\
        \coloredMidrule{white}{alternateRowColor}
        \rowcolor{alternateRowColor}
        WILDS Camelyon17 & 0 & 25 & 50 & 75 & 100 \\
        \coloredBelowRuleSep{white}
         \quad ARM       &  61.2 $\pm$ 5.2 &  59.5 $\pm$ 4.2 & 59.7 $\pm$ 4.2  & 59.7 $\pm$ 4.3 & 59.7 $\pm$ 4.2 \\
        \quad ARM$^{+}$ & 55.8  $\pm$ 0.8& 55.1 $\pm$ 1.7 & 55.0  $\pm$ 1.7& 55.0 $\pm$ 1.8 & 55.0 $\pm$ 1.8 \\
        \midrule
        \quad ERM       &  68.6   $\pm$ 7.8& 68.6 $\pm$ 7.8  &  68.6 $\pm$ 7.8 & 68.6 $\pm$ 7.8 & 68.6 $\pm$ 7.8 \\
        \quad ERM$^{+}$ & 50.1 $\pm$ 0.1 & 50.1 $\pm$ 0.1 & 50.1 $\pm$ 0.1 & 50.1  $\pm$ 0.1& 50.1 $\pm$ 0.1 \\
        \coloredMidrule{white}{alternateRowColor}
        \rowcolor{alternateRowColor}
        Tiny ImageNet-C & 0 & 25 & 50 & 75 & 100  \\
        \coloredBelowRuleSep{white}
        \quad ARM   & 8.2  $\pm$ 0.3 & 8.3 $\pm$ 0.3  &  8.2  $\pm$ 0.3& 8.3  $\pm$ 0.3& 8.2 $\pm$ 0.3 \\
        \quad ARM$^{+}$  & 1.9 $\pm$ 0.1 & 1.9 $\pm$ 0.1 & 1.9 $\pm$ 0.1 & 1.9 $\pm$ 0.1 & 1.9 $\pm$ 0.1  \\
        \midrule
        \quad ERM   & 9.5  $\pm$ 0.4 &  9.5  $\pm$ 0.4&  9.5  $\pm$ 0.4&  9.5 $\pm$ 0.4 & 9.5 $\pm$ 0.4  \\
        \quad ERM$^{+}$  & 8.3 $\pm$ 0.3 & 8.3 $\pm$ 0.3 & 8.3  $\pm$ 0.3& 8.3 $\pm$ 0.3 & 8.3 $\pm$ 0.3 \\
        \bottomrule
    \end{tabular}
    }
    \end{center}
    \label{table:sup_ablation_pluses_wo}
\end{table}